\documentclass{article}

\usepackage{amsmath}
\usepackage{amsthm}
\usepackage{amssymb}
\usepackage{graphicx}
\usepackage{subfigure}
\usepackage[numbers]{natbib}
\usepackage{algorithm}
\usepackage{algorithmic}
\usepackage[hidelinks]{hyperref}
\usepackage{authblk}
\usepackage{fullpage}
\usepackage{multirow}
\usepackage{natbib}

\newtheorem{theorem}{Theorem}
\newtheorem{corollary}[theorem]{Corollary}
\newtheorem{lemma}[theorem]{Lemma}
\newtheorem{definition}[theorem]{Definition}

\renewcommand{\eqref}[1]{Eq.~(\ref{#1})}
\newcommand{\figref}[1]{Fig.~\ref{#1}}

\title{Scaling SVM and Least Absolute Deviations via Exact Data Reduction}
\author[1]{Jie Wang}
\author[1]{Peter Wonka}
\author[1]{Jieping Ye}
\affil[1]{Computer Science and Engineering, Arizona State University,
            USA}

\begin{document}

\maketitle

\begin{abstract}
The support vector machine (SVM) is a widely used method for classification. Although many efforts have been devoted to develop efficient solvers, it remains challenging to apply SVM to large-scale problems. A nice property of SVM is that the non-support vectors have no effect on the resulting classifier. Motivated by this observation, we present fast and efficient screening rules to discard non-support vectors by analyzing the {\bf d}ual problem of SVM via {\bf v}ariational {\bf i}nequalities (DVI). As a result, the number of data instances to be entered into the optimization can be substantially reduced. Some appealing features of our screening method are: (1) DVI is safe in the sense that the vectors discarded by DVI are guaranteed to be non-support vectors; (2) the data set needs to be scanned only once to run the screening, whose computational cost is negligible compared to that of solving the SVM problem; (3) DVI is independent of the solvers and can be integrated with any existing efficient solvers. We also show that the DVI technique can be extended to detect non-support vectors in the least absolute deviations regression (LAD). To the best of our knowledge, there are currently no screening methods for LAD.
We have evaluated DVI on both synthetic and real data sets. Experiments indicate that DVI significantly outperforms the existing state-of-the-art screening rules for SVM, and is very effective in discarding non-support vectors for LAD. The speedup gained by DVI rules can be up to two orders of magnitude.
\end{abstract}
\section{Introduction}

The support vector machine is one of the most popular classification tools in machine learning. Many efforts have been devoted to develop efficient solvers for SVM \cite{Hastie2004,Joachims2006,Shalev-Shwartz2007,Hsieh2008,Fan2008}. However, the applications of SVM to large-scale problems still pose significant challenges. To address this issue, one promising approach is by ``{\it screening}". The key idea of screening is motivated by the well known feature of SVM, that is, the resulting classifier is determined only by the so called ``{\it support vectors}". If we first identify the non-support vectors via screening, and then remove them from the optimization, it may lead to substantial savings in the computational cost and memory. Another useful tool in machine learning and statistics is the least absolute deviations regression (LAD) \cite{Powell1984,Wang2006,Chen2008,Rao2008} or $\ell_1$ method. When the protection against outliers is a major concern, LAD provides a useful and plausible alternative to the classical least squares or $\ell_2$ method for linear regression. In this paper, we study both SVM and LAD under a unified framework.


The idea of screening has been successfully applied to a large class of $\ell_1$-regularized problems \cite{Ghaoui2012,Xiang2011,Wang2013,Tibshirani11}, including Lasso, $\ell_1$-regularized logistic regression, elastic net, and more general convex problems. Those methods are able to discard a large portion of ``inactive" features, which has $0$ coefficients in the optimal solution, and the speedup can be several orders of magnitude.

Recently, Ogawa et al. \cite{Ogawa2013} proposed a ``{\it {\bf s}afe {\bf s}creening}" rule to identify {\bf n}on-{\bf s}upport {\bf v}ectors for SVM; in this paper, we refer to this method as SSNSV for convenience. Notice that, the former approaches for $\ell_1$-regularized problems aim to discard inactive ``{\it features}", while SSNSV is to identify the non-support ``{\it vectors}". This essential difference makes SSNSV a nontrivial extension of the existing feature screening methods. Although there exist many methods for data reduction for SVM \cite{Achlioptas2002,Yu2003,Cao2006}, they are not safe, in the sense that the resulting classification model may be different.
To the best of our knowledge, SSNSV is the only existing safe screening method \cite{Ogawa2013} to identify the non-support vectors for SVM. However, in order to run the screening, SSNSV needs to iteratively determine an appropriate parameter value and an associated feasible solution, which can be very time consuming.

In this paper, we develop novel efficient and effective screening rules, called ``DVI", for a class of supervised learning problems including SVM and LAD \cite{Buchinsky1998,Jin2001}. The proposed method, DVI, shares the same advantage as SSNSV \cite{Ogawa2013}, that is, both rules are safe in the sense that the discarded vectors are guaranteed to be non-support vectors. The proposed DVI identifies the non-support vectors by estimating a lower bound of the inner product between each vector and the optimal solution, which is unknown. The more accurate the estimation is, the more non-support vectors can be detected. However, the estimation turns out to be highly non-trivial since the optimal solution is not available. To overcome this difficulty, we propose a novel framework to accurately estimate the optimal solution via the estimation of the ``dual optimal solution", as the primal and dual optimal solutions can be related by the KKT conditions \cite{Guler2010}. Our main technical contribution is to estimate the dual optimal solution via the so called ``variational inequalities" \cite{Guler2010}. Our experiments on both synthetic and real data demonstrate that DVI can identify far more non-support vectors than SSNSV. Moreover, by using the same technique, that is, variational inequalities, we can strictly improve SSNSV in identifying the non-support vectors. Our results also show that DVI is very effective in discarding non-support vectors for LAD. The speedup gained by DVI rules can be up to two orders of magnitude.

The rest of this paper is organized as follows. In Section \ref{section:basics}, we study the SVM and LAD problems under a unified framework. We then introduce our DVI rules in detail for the general formulation in Sections \ref{section:estimate_dual} and \ref{section:DVI_general}. In Sections \ref{section:DVI} and \ref{section:LAD}, we extend the DVI rules derived in Section \ref{section:DVI_general} to SVM and LAD respectively. In Section \ref{section:experiments}, we evaluate our DVI rules for SVM and LAD using both synthetic and real data. We conclude this paper in Section \ref{section:conclusion}.

{\bf Notations:} Throughout this paper, we use $\langle{\bf x}, {\bf y}\rangle=\sum_i x_i y_i$ to denote the inner product of vectors ${\bf x}$ and ${\bf y}$, and $\|{\bf x}\|^2=\langle{\bf x},{\bf x}\rangle$. For vector ${\bf x}$, let $[{\bf x}]_i$ be the $i^{th}$ component of ${\bf x}$. If ${\bf M}$ is a matrix, ${\bf m}_i$ is the $i^{th}$ column of ${\bf M}$ and $[{\bf M}]_{i,j}$ is the $(i,j)^{th}$ entry of ${\bf M}$. Given a scalar $x$, we denote $\max\{x,0\}$ by $[x]_+$. For the index set $\mathcal{I}:=\{1,\ldots,l\}$, let $\mathcal{J}:=\{j_1,\ldots,j_k\}\subseteq\mathcal{I}$ and $\mathcal{J}^{\rm c}:=\mathcal{I}\setminus\mathcal{J}$. For a vector ${\bf x}$ or a matrix ${\bf M}$, let $[{\bf x}]_{\mathcal{J}}=([{\bf x}]_{j_1},\ldots,[{\bf x}]_{j_k})^T$ and $[{\bf M}]_{\mathcal{J}}=({\bf m}_{j_1},\ldots,{\bf m}_{j_k})$. Moreover, let $\Gamma_0(\Re^n)$ be the class of proper and lower semicontinuous convex functions from $\Re^n$ to $(-\infty,\infty]$. The conjugate of $f\in\Gamma_0(\Re^n)$ is the function $f^*\in\Gamma_0(\Re^n)$ given by
\begin{equation}\label{eqn:conjugate}
f^*:\Re^n\rightarrow(-\infty,\infty]:\theta\mapsto\sup_{{\bf x}\in\Re^n}{\bf x}^T\theta-f({\bf x}).
\end{equation}
The biconjugate of $f\in\Gamma_0(\Re^n)$ is the function $f^{**}\in\Gamma_0(\Re^n)$ given by
\begin{equation}
f^{**}:\Re^n\rightarrow(-\infty,\infty]:{\bf x}\mapsto\sup_{\theta\in\Re^n}{\bf x}^T\theta-f^*(\theta).
\end{equation}

\section{Basics and Motivations}\label{section:basics}
In this section, we study the SVM and LAD problems under a unified framework. Then, we motivate the general screening rules via the KKT conditions.
Consider the convex optimization problems of the following form:
\begin{align}\label{prob:general}
\min_{{\bf w}\in\Re^n} \frac{1}{2}\|{\bf w}\|^2+C\Phi({\bf w}),
\end{align}
where $\Phi:\Re^n\rightarrow\Re$ is a convex function but not necessarily differentiable and $C>0$ is a regularization parameter. Notice that, the function $\Phi$ is generally referred to as the empirical loss. More specifically,
suppose we have a set of observations $\{{\bf x}_i, y_i\}_{i=1}^l$, where ${\bf x}_i\in\Re^n$ and $y_i\in\Re$ are the $i^{th}$ data instance and the corresponding response. We focus on the following function class:
\begin{align}
\Phi({\bf w})=\sum_{i=1}^l\varphi\left({\bf w}^T(a_i{\bf x}_i)+b_iy_i\right),
\end{align}
where $\varphi:\Re\rightarrow\Re_+$ is a nonconstant continuous sublinear function, and $a_i,b_i$ are scalars. We provide the definition of sublinear function as follows.
\begin{definition}\label{def:sublinear}
\cite{Hiriart-Urruty1993} A function $\sigma:\Re^n\rightarrow(-\infty,\infty]$ is said to be sublinear if it is convex, and positively homogeneous, i.e.,
\begin{align}
\sigma(tx)=t\sigma(x),\,\,\forall x\in\Re^n {\rm and}\,\,t>0.
\end{align}
\end{definition}
We will see that SVM and LAD are both special cases of problem (\ref{prob:general}).
A nice property of the function $\varphi$ is that the biconjugate $\varphi^{**}$ is exactly $\varphi$ itself, as stated in Lemma \ref{lemma:biconjugate}.
\begin{lemma}\label{lemma:biconjugate}
For the function $\varphi:\Re\rightarrow\Re_+$ which is continuous and sublinear, we have $\varphi\in\Gamma_0(\Re)$, and thus $\varphi^{**}=\varphi$.
\end{lemma}
It is straightforward to check the statement in Lemma \ref{lemma:biconjugate} by verifying the requirements of the function class $\Gamma_0(\Re)$. For self-completeness, we provide a proof in the supplement. According to Lemma \ref{lemma:biconjugate}, problem (\ref{prob:general}) can be rewritten as

\begin{align}\label{eqn:general_fenchel}
&\min_{{\bf w}\in\Re^n} \frac{1}{2}\|{\bf w}\|^2+C\sum_{i=1}^l\varphi^{**}\left({\bf w}^T(a_i{\bf x}_i)+b_iy_i\right)\\ \nonumber
=&\min_{{\bf w}\in\Re^n} \frac{1}{2}\|{\bf w}\|^2+C\sum_{i=1}^l\left\{\sup_{\theta_i\in\Re}\theta_i\left[{\bf w}^T(a_i{\bf x}_i)+b_iy_i\right]-\varphi^*(\theta_i)\right\}\\ \nonumber
=&\min_{{\bf w}\in\Re^n}\sup_{\substack{\theta_i\in\Re\\i=1,\ldots,l}}\frac{1}{2}\|{\bf w}\|^2+C\sum_{i=1}^l\left\{\theta_i\left[{\bf w}^T(a_i{\bf x}_i)+b_iy_i\right]-\varphi^*(\theta_i)\right\}\\ \nonumber
=&\sup_{\theta\in\Re^l}-C\sum_{i=1}^l\varphi^*(\theta_i)+\min_{{\bf w}\in\Re^n}\frac{1}{2}\|{\bf w}\|^2+C\langle{\bf Z}{\bf w}+\bar{\bf y},\theta\rangle,
\end{align}

where $\theta=(\theta_1,\ldots,\theta_l)^T$, ${\bf Z}=(a_i{\bf x}_i,\ldots,a_l{\bf x}_l)^T$ and $\bar{\bf y}=(b_1y_1,\ldots,b_ly_l)^T$.
Let $\ell({\bf w}):=\frac{1}{2}\|{\bf w}\|^2+C\langle{\bf Z}{\bf w}+\bar{\bf y},\theta\rangle$. The reason we can exchange the order of $\min$ and $\sup$ in \eqref{eqn:general_fenchel} is due to the strong duality of problem (\ref{prob:general}) \cite{Boyd04}.

By setting $\frac{\partial \ell({\bf w})}{\partial {\bf w}}=0$, we have
\begin{align}\label{eqn:w_theta0}
{\bf w}^*=-C{\bf Z}^T\theta,
\end{align}
and thus
\begin{align}
\min_{\bf w}\ell({\bf w})=\ell({\bf w}^*)=-\frac{C^2}{2}\|{\bf Z}^T\theta\|^2+C\langle\bar{\bf y},\theta\rangle.
\end{align}
Hence, \eqref{eqn:general_fenchel} becomes
\begin{align}\label{prob:fenchel_dual1}
\sup_{\theta\in\Re^l}-C\sum_{i=1}^l\varphi^*(\theta_i)-\frac{C^2}{2}\|{\bf Z}^T\theta\|^2+C\langle\bar{\bf y},\theta\rangle.
\end{align}
Moreover, because $\varphi\in\Gamma_0(\Re)$ is sublinear by Lemma \ref{lemma:biconjugate}, we know that $\varphi^*$ is the indicator function for a closed convex set. In fact, we have the following result:

\begin{lemma}\label{lemma:biconjugate_indicator}
For the nonconstant continuous sublinear function $\varphi:\Re\rightarrow\Re_+$, there exists a nonempty closed interval $I_{\varphi}=[\alpha,\beta]$ with $\alpha,\beta\in\Re$ and $\alpha<\beta$ such that
\begin{align}
\varphi^*(t):=\iota_{[\alpha,\beta]}=
\begin{cases}
0,\hspace{4mm}{\rm if}\hspace{2mm}t\in [\alpha,\beta],\\
\infty,\hspace{2mm}{\rm otherwise.}
\end{cases}
\end{align}
\end{lemma}

Let $I_{\varphi}^l=[\alpha, \beta]^l$. We can rewrite problem (\ref{prob:fenchel_dual1}) as
\begin{align}\label{prob:general_dual}
\sup_{\theta\in I_{\varphi}^l}-\frac{C^2}{2}\|{\bf Z}^T\theta\|^2+C\langle\bar{\bf y},\theta\rangle.
\end{align}
Problem (\ref{prob:general_dual}) is in fact the dual problem of (\ref{prob:general}). Moreover, the ``$\sup$" in problem (\ref{prob:general_dual}) can be replaced by ``$\max$" due to the strong duality \cite{Boyd04} of problem (\ref{prob:general}). Since $C>0$, problem (\ref{prob:general_dual}) is equivalent to
\begin{align}\label{prob:general_dual1}
\min_{\theta\in I_{\varphi}^l}\frac{C}{2}\|{\bf Z}^T\theta\|^2-\langle\bar{\bf y},\theta\rangle.
\end{align}
Let ${\bf w}^*(C)$ and $\theta^*(C)$ be the optimal solutions of (\ref{prob:general}) and (\ref{prob:general_dual}) respectively. \eqref{eqn:w_theta0} implies that
\begin{align}\label{eqn:primal_dual}
{\bf w}^*(C)=-C{\bf Z}^T\theta^*(C).
\end{align}
The KKT conditions\footnote{Please refer to the supplement for details.} of problem (\ref{prob:general_dual1}) are
\begin{align}\label{eqn:KKT_general}
[\theta^*(C)]_i\in
\begin{cases}
\beta,\hspace{7mm}{\rm if }\hspace{1mm}-\langle{\bf w}^*(C),a_i{\bf x}_i\rangle<b_iy_i;\\
[\alpha,\beta],\hspace{1mm}{\rm if }\hspace{1mm}-\langle{\bf w}^*(C),a_i{\bf x}_i\rangle=b_iy_i;\\
\alpha,\hspace{7mm}{\rm if}\hspace{1mm}-\langle{\bf w}^*(C),a_i{\bf x}_i\rangle>b_iy_i;\\
\end{cases}
\hspace{-4mm}i=1,\ldots,l.
\end{align}
For notational convenience, let
\begin{align*}
\mathcal{R}&=\{i:-\langle{\bf w}^*(C),a_i{\bf x}_i\rangle>b_iy_i\},\\
\mathcal{E}&=\{i:-\langle{\bf w}^*(C),a_i{\bf x}_i\rangle=b_iy_i\},\\
\mathcal{L}&=\{i:-\langle{\bf w}^*(C),a_i{\bf x}_i\rangle<b_iy_i\}.
\end{align*}
We call the vectors in the set $\mathcal{E}$ as ``support vectors". All the other vectors in $\mathcal{R}$ and $\mathcal{L}$ are called ``non-support vectors". The KKT conditions in (\ref{eqn:KKT_general}) imply that, if some of the data instances are known to be members of $\mathcal{R}$ and $\mathcal{L}$, then the corresponding components of $\theta^*(C)$ can be set accordingly and we only need the other components of $\theta^*(C)$. More precisely, we have the following result:

\begin{lemma}\label{lemma:dual_reduce}
Given index sets $\hat{\mathcal{R}}\subseteq\mathcal{R}$ and $\hat{\mathcal{L}}\subseteq\mathcal{L}$, we have

1. $[\theta^*(C)]_{\hat{\mathcal{R}}}=\alpha$ and $[\theta^*(C)]_{\hat{\mathcal{L}}}=\beta$.

2. Let $\hat{\mathcal{S}}=\hat{\mathcal{R}}\bigcup\hat{\mathcal{L}}$,  $|\hat{\mathcal{S}}^{\rm c}|$ be the cardinality of the set $\hat{\mathcal{S}}^{\rm c}$, $\hat{{\bf G}}_{11}=[{\bf Z}^T]_{\hat{\mathcal{S}}^{\rm c}}^T[{\bf Z}^T]_{\hat{\mathcal{S}}^{\rm c}}$, $\hat{{\bf G}}_{12}=[{\bf X}^T]_{\hat{\mathcal{S}}^{\rm c}}^T[{\bf X}^T]_{\hat{\mathcal{S}}}$ and $\hat{\bf y}={\bf y}_{\hat{\mathcal{S}}^{\rm c}}-C\hat{\bf G}_{12}[\theta^*(C)]_{\hat{\mathcal{S}}}$. Then,  $[\theta^*(C)]_{\hat{\mathcal{S}}^{\rm c}}$ can be computed by solving the following problem:
\begin{align}\label{prob:dual_reduced}
\min_{\hat{\theta}\in\Re^{|\hat{\mathcal{S}}^{\rm c}|}}\frac{C}{2}\hat{\theta}^T\hat{{\bf G}}_{11}\hat{\theta}-\hat{{\bf y}}^T\hat{\theta},\,\,{\rm s.t. }\,\,\hat{\theta}\in[\alpha,\beta]^{|\hat{\mathcal{S}}^{\rm c}|}.
\end{align}
\end{lemma}

Clearly, if $|\hat{\mathcal{S}}|$ is large compared to $|\mathcal{I}|=l$, the computational cost for solving problem (\ref{prob:dual_reduced}) can be much cheaper than solving the full problem (\ref{prob:general_dual1}). To determine the membership of the data instances, \eqref{eqn:primal_dual} and (\ref{eqn:KKT_general}) imply that
\begin{equation}\tag{R1}\label{rule1}
\hspace{-5mm}C\langle{\bf Z}^T\theta^*(C),a_i{\bf x}_i\rangle>b_iy_i\Rightarrow[\theta^*(C)]_i=\alpha\Leftrightarrow i\in\mathcal{R};
\end{equation}
\begin{equation}\tag{R2}\label{rule2}
\hspace{-5mm}C\langle{\bf Z}^T\theta^*(C),a_i{\bf x}_i\rangle<b_iy_i\Rightarrow[\theta^*(C)]_i=\beta\Leftrightarrow i\in\mathcal{L}.
\end{equation}
However, (\ref{rule1}) and (\ref{rule2}) are generally not applicable since $\theta^*(C)$ is unknown. To overcome this difficulty, we can estimate a region $\Theta$ such that $\theta^*(C)\in\Theta$. As a result, we obtain the relaxed version of (\ref{rule1}) and (\ref{rule2}):
\begin{equation}\tag{R1$'$}\label{rule1'}
\hspace{-6mm}\min_{\theta\in\Theta}C\langle{\bf Z}^T\theta,a_i{\bf x}_i\rangle>b_iy_i\Rightarrow[\theta^*(C)]_i=\alpha\Leftrightarrow i\in\mathcal{R};
\end{equation}
\begin{equation}\tag{R2$'$}\label{rule2'}
\hspace{-6mm}\max_{\theta\in\Theta}C\langle{\bf Z}^T\theta,a_i{\bf x}_i\rangle<b_iy_i\Rightarrow[\theta^*(C)]_i=\beta\Leftrightarrow i\in\mathcal{L}.
\end{equation}
Notice that, (\ref{rule1'}) and (\ref{rule2'}) serve as the foundation of the proposed DVI rules and the method in \cite{Ogawa2013}. In the subsequent sections, we first estimate the region $\Theta$ which includes $\theta^*(C)$, and then derive the screening rules based on (\ref{rule1'}) and (\ref{rule2'}).

\noindent{\bf Method to solve problem (\ref{prob:dual_reduced})} 

It is known that, problem (\ref{prob:dual_reduced}) can be efficiently solved by the dual coordinate descent method \cite{Hsieh2008}. More precisely, the optimization procedure starts from an initial point $\hat{\theta}^0\in\Re^{|\hat{\mathcal{S}}^{\rm c}|}$ and generates a sequence of points $\{\hat{\theta}^k\}_{k=0}^{\infty}$. The process from $\hat{\theta}^k$ to $\hat{\theta}^{k+1}$ is referred to as an outer iteration. In each outer iteration, we update the components of $\hat{\theta}^k$ one at a time and thus get a sequence of points $\hat{\theta}^{k,i}\in\Re^{|\hat{\mathcal{S}}^{\rm c}|}$, $i=1,\ldots,|\hat{\mathcal{S}}^{\rm c}|$. Suppose we are at the $k^{th}$ outer iteration. To get $\hat{\theta}^{k,i}$ from $\hat{\theta}^{k,i-1}$, we need to solve the following optimization problem:

\begin{align}\label{prob:dual_CD}
\min_{t}\,\,\,&\frac{C}{2}(\hat{\theta}^{k,i-1}+t{\bf e}_{i})^T\hat{\bf G}_{11}(\hat{\theta}^{k,i-1}+t{\bf e}_{i})-\hat{\bf y}^T(\hat{\theta}^{k,i-1}+t{\bf e}_{i})\\ \nonumber
{\rm s.t.}\,\,\,&[\hat{\theta}^{k,i-1}]_{i}+t\in[\alpha,\beta],\,\,i=1,\ldots,l,
\end{align}

where ${\bf e}_{i}=(0,\ldots,1,\ldots,0)^T$.
Clearly, problem (\ref{prob:dual_CD}) is equivalent to the following 1D optimization problem:
\begin{align}\label{prob:dual_CD_1d}
\min_{t}\,\,\,&\frac{C}{2}[\hat{\bf G}_{11}]_{i,i}t^2+(C{\bf e}_{i}^T\hat{\bf G}_{11}\hat{\theta}^{k,i-1}-[\hat{\bf y}]_{i})t\\ \nonumber
{\rm s.t.}\,\,\,&[\hat{\theta}^{k,i-1}]_{i}+t\in[\alpha,\beta],
\end{align}
which admits a closed form solution $t^*$.
Once $t^*$ is available, we can set $\hat{\theta}^{k,i}=\hat{\theta}^{k,i-1}+t^*{\bf e}_i$. For more details, please refer to \cite{Hsieh2008}.

In Section \ref{section:estimate_dual}, we first give an accurate estimation of the set $\Theta$ which includes $\theta^*(C)$ as in (\ref{rule1'}) and (\ref{rule2'}) via the variational inequalities. Then in Section \ref{section:DVI_general}, we present the novel DVI rules for problem (\ref{prob:general}) in detail.

\section{Estimation of the Dual Optimal Solution}\label{section:estimate_dual}

For problem (\ref{prob:general_dual1}), suppose we are given two parameter values $0<C_0<C$ and $\theta^*(C_0)$ is known. Then, Theorem \ref{thm:bound_svm} shows that $\theta^*(C)$ can be effectively bounded in terms of $\theta^*(C_0)$. The main technique we use is the so called variational inequalities. For self-completeness, we cite the definition of variational inequalities as follows.

\begin{theorem}\label{thm:vi}
\cite{Guler2010} Let $A\subseteq\Re^n$ be a convex set, and let $h$ be a G$\hat{a}$teaux differentiable function on an open set containing $A$. If ${\bf x}^*$ is a local minimizer of $h$ on $A$, then
\begin{align}
\langle\nabla h({\bf x}^*),{\bf x}-{\bf x}^*\rangle\geq0,\hspace{2mm}\forall {\bf x}\in A.
\end{align}
\end{theorem}

Via the variational inequalities, the following theorem shows that $\theta^*(C)$ can estimated in terms of $\theta^*(C_0)$.

\begin{theorem}\label{thm:bound_svm}
For problem (\ref{prob:general_dual1}), let $C>C_0>0$. Then
\begin{align*}
\|{\bf Z}^T\theta^*(C)-\tfrac{C_0+C}{2C}{\bf Z}^T\theta^*(C_0)\|\leq\tfrac{C-C_0}{2C}\|{\bf Z}^T\theta^*(C_0)\|.
\end{align*}
\end{theorem}

\begin{proof}
Let $g(\theta)$ be the objective function of problem (\ref{prob:general_dual1}). The variational inequality implies that
\begin{equation}\label{ineqn:vi1}
\langle\nabla g(\theta^*(C_0)),\theta-\theta^*(C_0)\rangle\geq0,\,\,\forall \theta\in[\alpha,\beta]^l;
\end{equation}
\begin{equation}\label{ineqn:vi2}
\langle\nabla g(\theta^*(C)),\theta-\theta^*(C)\rangle\geq0,\,\,\forall \theta\in[\alpha,\beta]^l.
\end{equation}
Notice that $\nabla g(\theta) = C{\bf Z}{\bf Z}^T\theta-\bar{\bf y}$, and $\theta^*(C_0)\in[\alpha,\beta]^l$ and $\theta^*(C)\in[\alpha,\beta]^l$. Plugging $\nabla g(\theta^*(C))$ and $\nabla g(\theta^*(C_0))$ into (\ref{ineqn:vi1}) and (\ref{ineqn:vi2}) leads to

\begin{equation}\label{ineqn:vi1'}
\langle C_0{\bf Z}{\bf Z}^T\theta^*(C_0)-\bar{\bf y},\theta^*(C)-\theta^*(C_0)\rangle\geq0;
\end{equation}
\begin{equation}\label{ineqn:vi2'}
\langle C{\bf Z}{\bf Z}^T\theta^*(C)-\bar{\bf y},\theta^*(C_0)-\theta^*(C)\rangle\geq0.
\end{equation}
We can see that the inequality in (\ref{ineqn:vi2'}) is equivalent to
\begin{equation}\label{ineqn:vi2''}
\langle \bar{\bf y}-C{\bf Z}{\bf Z}^T\theta^*(C),\theta^*(C)-\theta^*(C_0)\rangle\geq0.
\end{equation}
Then the statement follows by adding the inequalities in (\ref{ineqn:vi1'}) and (\ref{ineqn:vi2''}) together.
\end{proof}

\section{The Proposed DVI Rules}\label{section:DVI_general}

Given $C>C_0>0$ and $\theta^*(C_0)$, we can estimate $\theta^*(C)$ via Theorem \ref{thm:bound_svm}. Combining (\ref{rule1'}), (\ref{rule2'}) and Theorem (\ref{thm:bound_svm}), we develop the basic screening rule for problem (\ref{prob:general}) as summarized in the following theorem:

\begin{theorem}\label{thm:basic_dvi}
{\rm({\bf DVI})} For problem (\ref{prob:general_dual1}), suppose we are given $\theta^*(C_0)$. Then, for any $C>C_0$, we have $[\theta^*(C)]_i=\alpha$, i.e., $i\in\mathcal{R}$, if the following holds
\begin{align*}
&\tfrac{C+C_0}{2}\langle{\bf Z}^T\theta^*(C_0),a_i{\bf x}_i\rangle-\tfrac{C-C_0}{2}\|{\bf Z}^T\theta^*(C_0)\|\|a_i{\bf x}_i\|>b_iy_i.
\end{align*}
Similarly, we have $[\theta^*(C)]_i=\beta$, i.e., $i\in\mathcal{L}$, if
\begin{align*}
&\tfrac{C+C_0}{2}\langle{\bf Z}^T\theta^*(C_0),a_i{\bf x}_i\rangle+\tfrac{C-C_0}{2}\|{\bf Z}^T\theta^*(C_0)\|\|a_i{\bf x}_i\|<b_iy_i.
\end{align*}
\end{theorem}
\begin{proof}
We will prove the first half of the statement. The second half can be proved analogously.
To show $[\theta^*(C)]_i=\alpha$, i.e., $i\in\mathcal{R}$, (\ref{rule1}) implies that we only need to show $C\langle{\bf Z}^T\theta^*(C),a_i{\bf x}_i\rangle>b_iy_i$. Thus, we can see that
\begin{align*}
C\langle{\bf Z}^T\theta^*(C),a_i{\bf x}_i\rangle=&C\left\langle{\bf Z}^T\theta^*(C)-\tfrac{C_0+C}{2C}{\bf Z}^T\theta^*(C_0),a_i{\bf x}_i\right\rangle + C\left\langle\tfrac{C_0+C}{2C}{\bf Z}^T\theta^*(C_0),a_i{\bf x}_i\right\rangle\\
\geq&\tfrac{C_0+C}{2}\langle{\bf Z}^T\theta^*(C_0),a_i{\bf x}_i\rangle - C\left\|{\bf Z}^T\theta^*(C)-\tfrac{C_0+C}{2C}{\bf X}^T\theta^*(C_0)\right\|\|a_i{\bf x}_i\|\\
\geq&\tfrac{C_0+C}{2}\langle{\bf Z}^T\theta^*(C_0),a_i{\bf x}_i\rangle-\tfrac{C-C_0}{2}\|{\bf Z}^T\theta^*(C_0)\|\|a_i{\bf x}\|\\
>&b_iy_i.
\end{align*}

Note that, the second inequality is due to Theorem \ref{thm:bound_svm}, and the last line is due to the statement. This completes the proof.
\end{proof}

In real applications, the optimal parameter value of $C$ is unknown and we need to estimate it. Commonly used model selection strategies such as cross validation and stability selection need to solve the optimization problems over a grid of turning parameters $0<C_1<C_2<\ldots<C_{\mathcal{K}}$ to determine an appropriate value for $C$. This procedure is usually very time consuming, especially for large scale problems. To this end, we propose a sequential version of the proposed DVI below.

\begin{corollary}\label{corollary:sdvi_g}
{\rm ({\bf DVI$_s^*$})} For problem (\ref{prob:general_dual1}), suppose we are given a sequence of parameters $0<C_1<C_2<\ldots<C_{\mathcal{K}}$. Assume $\theta^*(C_k)$ is known for an arbitrary integer $1\leq k<\mathcal{K}$. Then, for $C_{k+1}$, we have $[\theta^*(C_{k+1})]_i=\alpha$, i.e., $i\in\mathcal{R}$, if the following holds
\begin{align*}
&\tfrac{C_{k+1}+C_k}{2}\langle{\bf Z}^T\theta^*(C_k),a_i{\bf x}_i\rangle-\tfrac{C_{k+1}-C_k}{2}\|{\bf Z}^T\theta^*(C_k)\|\|a_i{\bf x}_i\|>b_iy_i.
\end{align*}
Similarly, we have $[\theta^*(C_{k+1})]_i=\beta$, i.e., $i\in\mathcal{L}$, if
\begin{align*}
&\tfrac{C_{k+1}+C_k}{2}\langle{\bf Z}^T\theta^*(C_k),a_i{\bf x}_i\rangle+\tfrac{C_{k+1}-C_k}{2}\|{\bf Z}^T\theta^*(C_k)\|\|a_i{\bf x}_i\|<b_iy_i.
\end{align*}
\end{corollary}

The main computational cost of {\rm \bf DVI$_s^*$} is due to the evaluation of $\langle{\bf Z}^T\theta^*(C_k),a_i{\bf x}_i\rangle$, $\|{\bf Z}^T\theta^*(C_k)\|$ and $\|a_i{\bf x}_i\|$. Let ${\bf G}={\bf Z}{\bf Z}^T$. It is easy to see that
\begin{align*}
&\langle{\bf Z}^T\theta^*(C_k),a_i{\bf x}_i\rangle = {\bf g}_i^T\theta^*(C_k),\\
&\|{\bf Z}^T\theta^*(C_k)\|^2=\theta^*(C_k)^T{\bf G}\theta^*(C_k),\\
&\|\bar{\bf x}_i\|^2=[{\bf G}]_{i,i}.
\end{align*}
where ${\bf g}_i$ is the $i^{th}$ column of ${\bf G}$. Since ${\bf G}$ is independent of $C_k$, it can be computed only once and thus the computational cost of {\rm \bf DVI$_s^*$} reduces to $O(l^2)$ to scan the entire data set. Indeed, by noting \eqref{eqn:primal_dual}, we can reconstruct DVI rules without the explicit computation of ${\bf G}$.

\begin{corollary}\label{corollary:sdvi_p_g}
{\rm ({\bf DVI$_s$})} For problem (\ref{prob:general}), suppose we are given a sequence of parameters $0<C_1<C_2<\ldots<C_{\mathcal{K}}$. Assume ${\bf w}^*(C_k)$ is known for an arbitrary integer $1\leq k<\mathcal{K}$. Then, for $C_{k+1}$, we have $[\theta^*(C_{k+1})]_i=\alpha$, i.e., $i\in\mathcal{R}$, if the following holds
\begin{align*}
&-\tfrac{C_k+C_{k+1}}{2C_k}\langle{\bf w}^*(C_k),a_i{\bf x}_i\rangle-\tfrac{C_{k+1}-C_k}{2C_k}\|{\bf w}^*(C_k)\|\|a_i{\bf x}_i\|>b_iy_i.
\end{align*}
Similarly, we have $[\theta^*(C_{k+1})]_i=\beta$, i.e., $i\in\mathcal{L}$, if
\begin{align*}
&-\tfrac{C_k+C_{k+1}}{2C_k}\langle{\bf w}^*(C_k),a_i{\bf x}_i\rangle+\tfrac{C_{k+1}-C_k}{2C_k}\|{\bf w}^*(C_k)\|\|a_i{\bf x}_i\|<b_iy_i.
\end{align*}
\end{corollary}

\section{Screening Rules for SVM}\label{section:DVI}

In Section \ref{subsection:DVI_SVM}, we first present the sequential DVI rules for SVM based on the results in Section \ref{section:DVI_general}. Then, in Section \ref{subsection:improve}, we show how to strictly improve SSNSV \cite{Ogawa2013} by the same technique used in DVI.
%

\subsection{DVI rules for SVM}\label{subsection:DVI_SVM}

Given a set of observations $\{{\bf x}_i,y_i\}_{i=1}^l$, where ${\bf x}_i$ and $y_i\in\{1,-1\}$ are the $i^{th}$ data instance and the corresponding class label,
the SVM takes the form of:
\begin{align}\label{prob:primal}
\min_{{\bf w}}\frac{1}{2}\|{\bf w}\|^2 + C\sum_{i=1}^l\left[1-{\bf w}^T(y_i{{\bf x}}_i)\right]_+.
\end{align}
It is easy to see that, if we set $\varphi(t)=[t]_+$ and $-a_i=b_i=y_i$, problem (\ref{prob:general}) becomes the SVM problem. To construct the DVI rules for SVM by Corollaries \ref{corollary:sdvi_g} and \ref{corollary:sdvi_p_g}, we only need to find $\alpha$ and $\beta$. In fact, we have the following result:

\begin{lemma}\label{lemma:biconjugate_svm}
Let $\varphi(t)=[t]_+$, then $\alpha=0$ and $\beta=1$, i.e.,
\begin{align}
\varphi^*(s)=\iota_{[0,1]}.
\end{align}
\end{lemma}

We omit the proof of Lemma \ref{lemma:biconjugate_svm} since it is a direct application of \eqref{eqn:conjugate}.
Then, we immediately have the following screening rules for the SVM problem. (For notational convenience, let $\bar{\bf x}_i=y_i{\bf x}_i$ and $\overline{\bf X}=(\bar{\bf x}_1,\ldots,\bar{\bf x}_l)^T$.)

\begin{corollary}\label{corollary:sdvi}
{\rm ({\bf DVI$_s^*$} for {\bf SVM})} For problem (\ref{prob:primal}), suppose we are given a sequence of parameters $0<C_1<C_2<\ldots<C_{\mathcal{K}}$. Assume $\theta^*(C_k)$ is known for an arbitrary integer $1\leq k<\mathcal{K}$. Then, for $C_{k+1}$, we have $[\theta^*(C_{k+1})]_i=0$, i.e., $i\in\mathcal{R}$, if the following holds
\begin{align*}
\hspace{-2mm}&\tfrac{C_{k+1}+C_k}{2}\langle\overline{\bf X}^T\theta^*(C_k),\bar{\bf x}_i\rangle-\tfrac{C_{k+1}-C_k}{2}\|\overline{\bf X}^T\theta^*(C_k)\|\|\bar{\bf x}_i\|>1.
\end{align*}
Similarly, we have $[\theta^*(C_{k+1})]_i=1$, i.e., $i\in\mathcal{L}$, if
\begin{align*}
\hspace{-2mm}&\tfrac{C_{k+1}+C_k}{2}\langle\overline{\bf X}^T\theta^*(C_k),\bar{\bf x}_i\rangle+\tfrac{C_{k+1}-C_k}{2}\|\overline{\bf X}^T\theta^*(C_k)\|\|\bar{\bf x}_i\|<1.
\end{align*}
\end{corollary}


\begin{corollary}\label{corollary:sdvi_p}
{\rm ({\bf DVI$_s$} for {\bf SVM})} For problem (\ref{prob:primal}), suppose we are given a sequence of parameters $0<C_1<C_2<\ldots<C_{\mathcal{K}}$. Assume ${\bf w}^*(C_k)$ is known for an arbitrary integer $1\leq k<\mathcal{K}$. Then, for $C_{k+1}$, we have $[\theta^*(C_{k+1})]_i=0$, i.e., $i\in\mathcal{R}$, if the following holds
\begin{align*}
&\tfrac{C_k+C_{k+1}}{2C_k}\langle{\bf w}^*(C_k),\bar{\bf x}_i\rangle-\tfrac{C_{k+1}-C_k}{2C_k}\|{\bf w}^*(C_k)\|\|\bar{\bf x}_i\|>1.
\end{align*}
Similarly, we have $[\theta^*(C_{k+1})]_i=1$, i.e., $i\in\mathcal{L}$, if
\begin{align*}
&\tfrac{C_k+C_{k+1}}{2C_k}\langle{\bf w}^*(C_k),\bar{\bf x}_i\rangle+\tfrac{C_{k+1}-C_k}{2C_k}\|{\bf w}^*(C_k)\|\|\bar{\bf x}_i\|<1.
\end{align*}
\end{corollary}


\subsection{Improving the existing method}\label{subsection:improve}

In the rest of this section, we briefly describe how to strictly improve SSNSV \cite{Ogawa2013} by using the same technique used in DVI rules (please refer to the supplement for more details). In view of \eqref{eqn:primal_dual}, (\ref{rule1'}) and (\ref{rule2'}) can be rewritten as:
\begin{equation}\tag{R1$''$}\label{rule1''}
\min_{{\bf w}\in\Omega}\langle{\bf w},\bar{\bf x}_i\rangle>1\Rightarrow[\theta^*(C)]_i=0\Leftrightarrow i\in\mathcal{R},
\end{equation}
\begin{equation}\tag{R2$''$}\label{rule2''}
\max_{{\bf w}\in\Omega}\langle{\bf w},\bar{\bf x}_i\rangle<1\Rightarrow[\theta^*(C)]_i=1\Leftrightarrow i\in\mathcal{L},
\end{equation}
where $\Omega$ is a set which includes ${\bf w}^*(C)$ (notice that, we have already set $-a_i=b_i=y_i$, $\alpha=0$ and $\beta=1$). It is easy to see that, the smaller $\Omega$ is, the tighter the bounds are in (\ref{rule1''}) and (\ref{rule2''}). Thus, more data instances' membership can be identified.

{\bf Estimation of ${\bf w}^*$ in SSNSV}

In \cite{Ogawa2013}, the authors consider the following equivalent formulation of SVM:
\begin{align}
\min_{\bf w}\frac{1}{2}\|{\bf w}\|^2,\hspace{1mm}{\rm s.t.}\hspace{1mm}\sum_{i=1}^l[1-y_i{\bf w}^T{\bf x}_i]_+\leq s
\end{align}
Let $\mathcal{F}_s=\{{\bf w}:\sum_{i=1}^l[1-y_i{\bf w}^T{\bf x}_i]_+\leq s\}$. Suppose we have two scalars $s_a>s_b>0$, and $\mathcal{F}_{s_b}\neq\emptyset$, $\hat{\bf w}(s_b)\in\mathcal{F}_{s_b}$. Then for $s\in[s_b,s_a]$, ${\bf w}^*(s)$ is inside the following region:
\begin{align}\label{eqn:jan_bound}
\Omega_{[s_b,s_a]}:=\left\{{\bf w}:
\begin{array}{l}
\langle{\bf w}^*(s_a),{\bf w}-{\bf w}^*(s_a)\rangle\geq0,\\
\|{\bf w}\|^2\leq\|\hat{\bf w}(s_b)\|^2
\end{array}
\right\}
\end{align}

{\bf Estimation of ${\bf w}^*$ via VI}

By using the same technique as in DVI, we can conclude that ${\bf w}^*(s)$
is inside the region:
\begin{align}\label{eqn:jan_bound_vi}
\Omega'_{[s_b,s_a]}:=\left\{{\bf w}:
\begin{array}{l}
\langle{\bf w}^*(s_a),{\bf w}-{\bf w}^*(s_a)\rangle\geq0,\\
\|{\bf w}-\frac{1}{2}\hat{\bf w}(s_b)\|\leq\frac{1}{2}\|\hat{\bf w}(s_b)\|
\end{array}
\right\}
\end{align}
We can see that $\Omega'_{[s_b,s_a]}\subset\Omega_{[s_b,s_a]}$, and thus SSNSV can be strictly improved by the estimation in (\ref{eqn:jan_bound_vi}). The rule based on $\Omega'[s_b,s_a]$ is presented in Theorem \ref{thm:snsv_i} in the supplement, which is call the ``enhanced" SSNSV (ESSNSV).

\section{Screening Rules for LAD}\label{section:LAD}

In this section, we extend DVI rules in Section \ref{section:DVI_general} to the least absolute deviations regression (LAD). Suppose we have a training set $\{{\bf x}_i,y_i\}_{i=1}^l$, where ${\bf x}_i\in\Re^n$ and $y_i\in\Re$.
The LAD problem takes the form of

\begin{equation}\label{prob:primal_rls}
\min_{\bf w}\frac{1}{2}\|{\bf w}\|^2+C\sum_{i=1}^l|y_i-{\bf w}^T{\bf x}_i|.
\end{equation}

We can see that, if we set $\varphi(t)=|t|$ and $-a_i=b_i=1$, problem (\ref{prob:general}) becomes the LAD problem. To construct the DVI rules for LAD based on Corollaries \ref{corollary:sdvi_g} and \ref{corollary:sdvi_p_g}, we need to find $\alpha$ and $\beta$.
Indeed, we have the following result:

\begin{lemma}\label{lemma:biconjugate_rls}
Let $\varphi(t)=|t|$, then $\alpha=-1$ and $\beta=1$, i.e.,
\begin{align}
\varphi^*(s)=\iota_{[-1,1]}.
\end{align}
\end{lemma}
We again omit the proof of Lemma \ref{lemma:biconjugate_rls} since it is a direct application of \eqref{eqn:conjugate}.
Then, it is straightforward to derive the sequential DVI rules for the LAD problem.

\begin{corollary}\label{corollary:sdvi_rls}
{\rm ({\bf DVI}$_s^*$ for {\bf LAD})} For problem (\ref{prob:primal_rls}), suppose we are given a sequence of parameter values $0<C_1<C_2<\ldots<C_{\mathcal{K}}$. Assume $\theta^*(C_k)$ is known for an arbitrary integer $1\leq k<\mathcal{K}$. Then, for $C_{k+1}$, we have $[\theta^*(C_{k+1})]_i=-1$ or $1$, i.e., $i\in\mathcal{R}$ or $i\in\mathcal{L}$, if the following holds respectively
\begin{align*}
{\rm 1.}\hspace{2mm}&\tfrac{C_{k+1}+C_k}{2}\langle{\bf X}^T\theta^*(C_k),{\bf x}_i\rangle-\tfrac{C_{k+1}-C_k}{2}\|{\bf X}^T\theta^*(C_k)\|\|{\bf x}_i\|>y_i.\\
{\rm 2.}\hspace{2mm}&\tfrac{C_{k+1}+C_k}{2}\langle{\bf X}^T\theta^*(C_k),{\bf x}_i\rangle+\tfrac{C_{k+1}-C_k}{2}\|{\bf X}^T\theta^*(C_k)\|\|{\bf x}_i\|<y_i.
\end{align*}
\end{corollary}

\begin{corollary}\label{corollary:sdvi_p_rls}
{\rm ({\bf DVI}$_s$ for {\bf LAD})} For problem (\ref{prob:primal_rls}), suppose we are given a sequence of parameter values $0<C_1<C_2<\ldots<C_{\mathcal{K}}$. Assume ${\bf w}^*(C_k)$ is known for an arbitrary integer $1\leq k<\mathcal{K}$. Then, for $C_{k+1}$, we have $[\theta^*(C_{k+1})]_i=-1$ or $1$, i.e., $i\in\mathcal{R}$ or $i\in\mathcal{L}$, if the following holds respectively
\begin{align*}
{\rm 1.}\hspace{2mm}&\tfrac{C_{k+1}+C_k}{2C_k}\langle{\bf w}^*(C_k),{\bf x}_i\rangle-\tfrac{C_{k+1}-C_k}{2C_k}\|{\bf w}^*(C_k)\|\|{\bf x}_i\|>y_i,\\
{\rm 2.}\hspace{2mm}&\tfrac{C_{k+1}+C_k}{2C_k}\langle{\bf w}^*(C_k),{\bf x}_i\rangle+\tfrac{C_{k+1}-C_k}{2C_k}\|{\bf w}^*(C_k)\|\|{\bf x}_i\|<y_i.
\end{align*}
\end{corollary}
To the best of our knowledge, ours are the first screening rules for LAD.

\section{Experiments}\label{section:experiments}

We evaluate DVI rules on both synthetic and real data sets. To measure the performance of the screening rules, we compute the rejection rate, that is, the ratio between the number of data instances whose membership can be identified by the rules and the total number of data instances. We test the rules along a sequence of $100$ parameters of $C\in[10^{-2},10]$ equally spaced in the logarithmic scale.


In Section \ref{subsection:experiment_DVI}, we compare the performance of DVI rules with SSNSV \cite{Ogawa2013}, which is the only existing method for identifying non-support vectors in SVM. Notice that, both of DVI rules and SSNSV are safe in the sense that no support vectors will be mistakenly discarded. We then evaluate DVI rules for LAD in Section \ref{subsection:experiment_LAD}.

\subsection{DVI for SVM}\label{subsection:experiment_DVI}

In this experiment, we first apply DVI$_s$ to three simple 2D synthetic data sets to illustrate the effectiveness of the proposed screening methods. Then we compare the performance of DVI$_s$, SSNSV and ESSNSV on: (a) IJCNN1 data set \cite{Prokhorov2001}; (b) Wine Quality data set \cite{Cortez2009}; (c) Forest Covertype data set \cite{Hettich1999}. The original Forest Covertype data set includes $7$ classes. We randomly pick two of the seven classes to construct the data set used in this paper.

\begin{figure*}[ht!]
\centering{
\includegraphics[width=0.3\columnwidth]{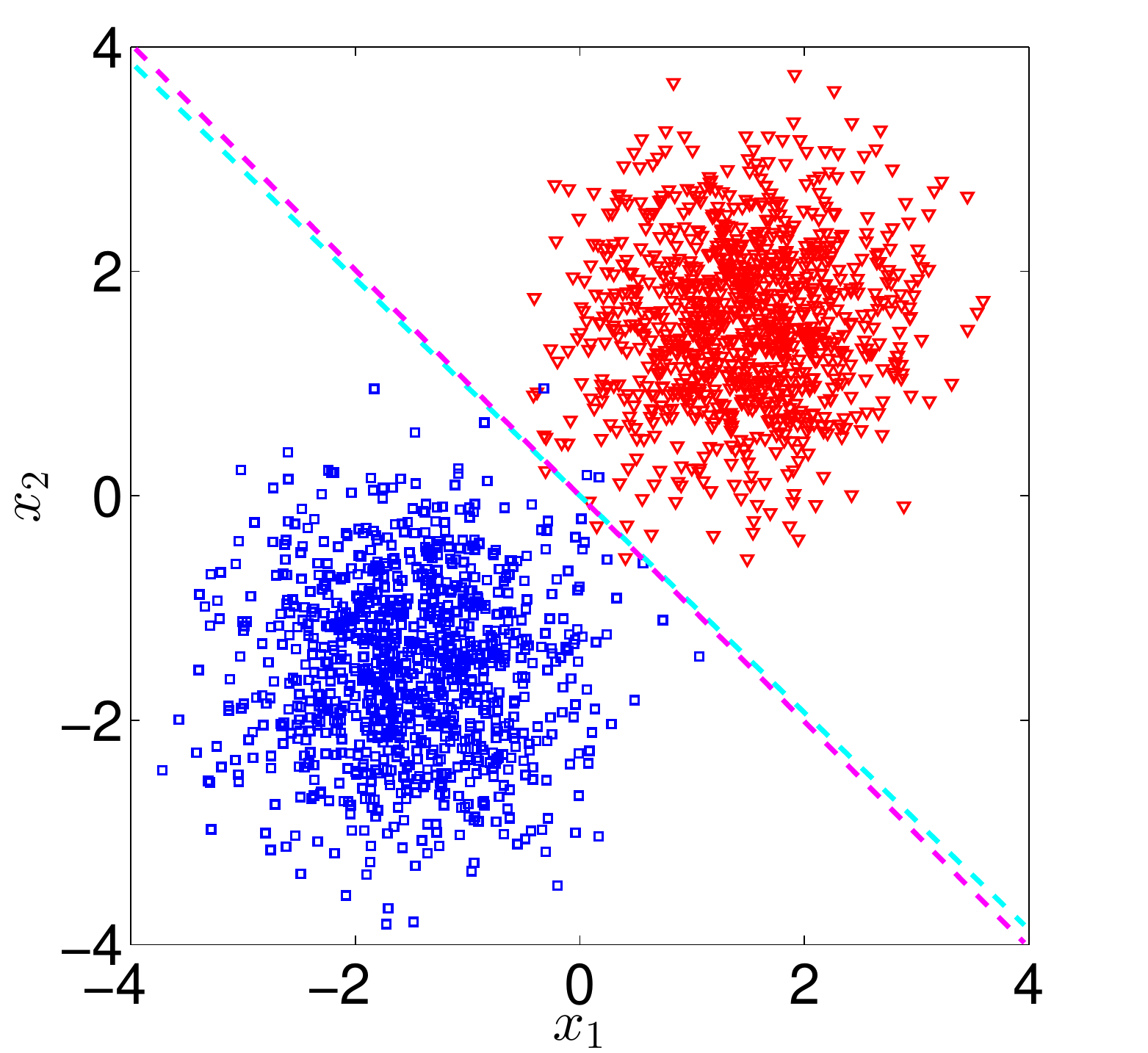}
\includegraphics[width=0.3\columnwidth]{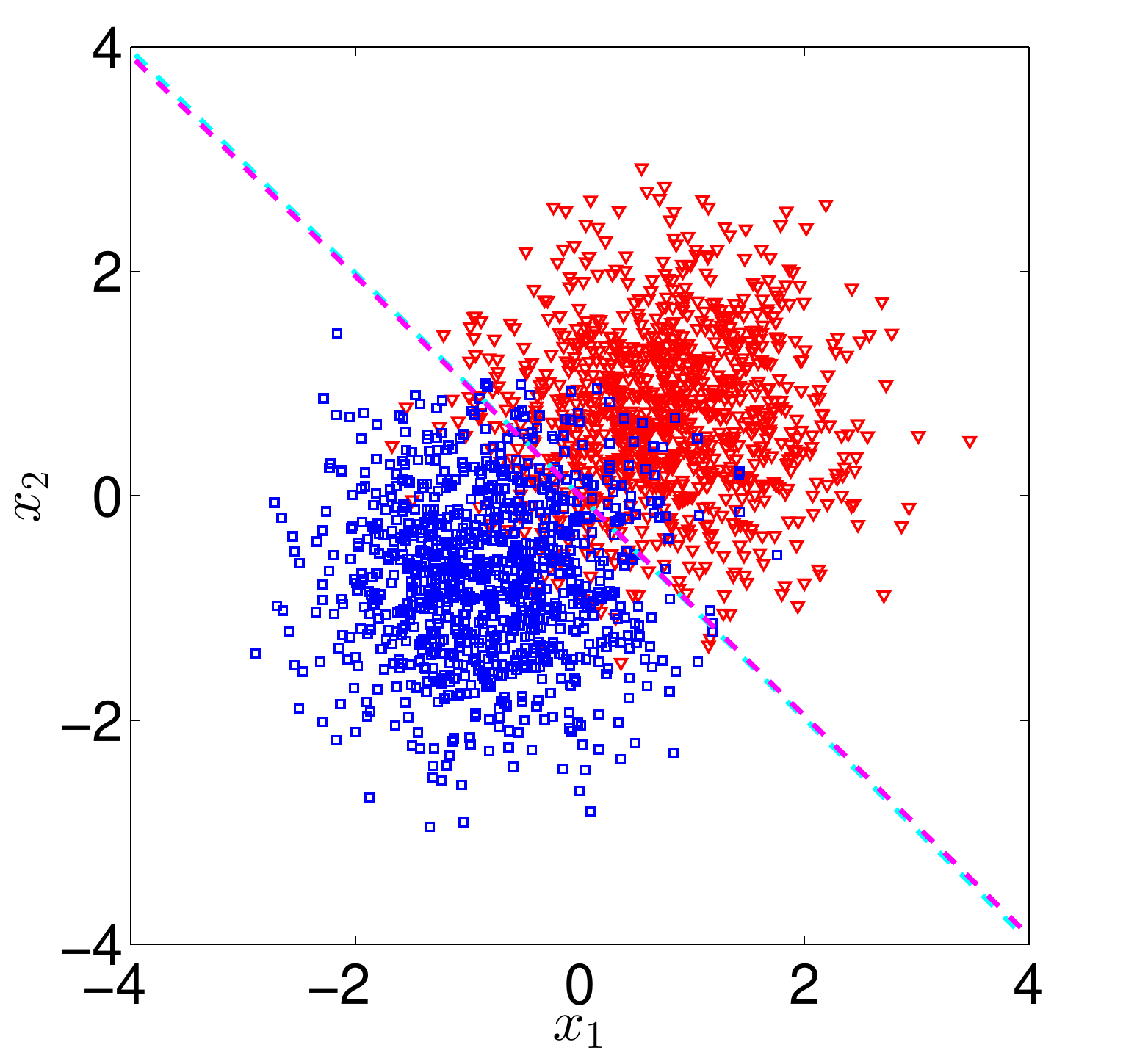}
\includegraphics[width=0.3\columnwidth]{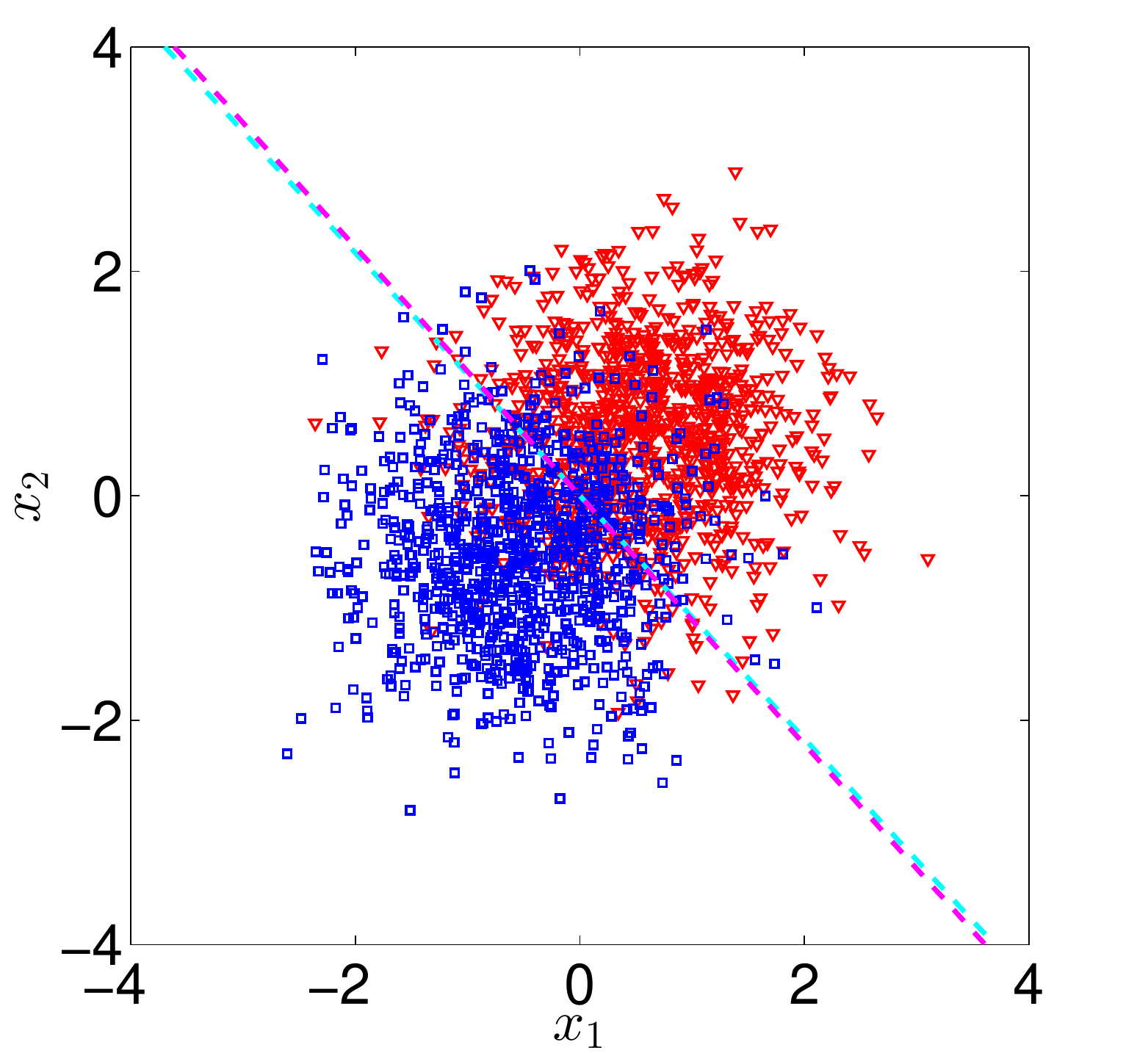}
\\
\subfigure[Toy1] { \label{fig:q175e1}
\includegraphics[width=0.22\columnwidth,angle=90]{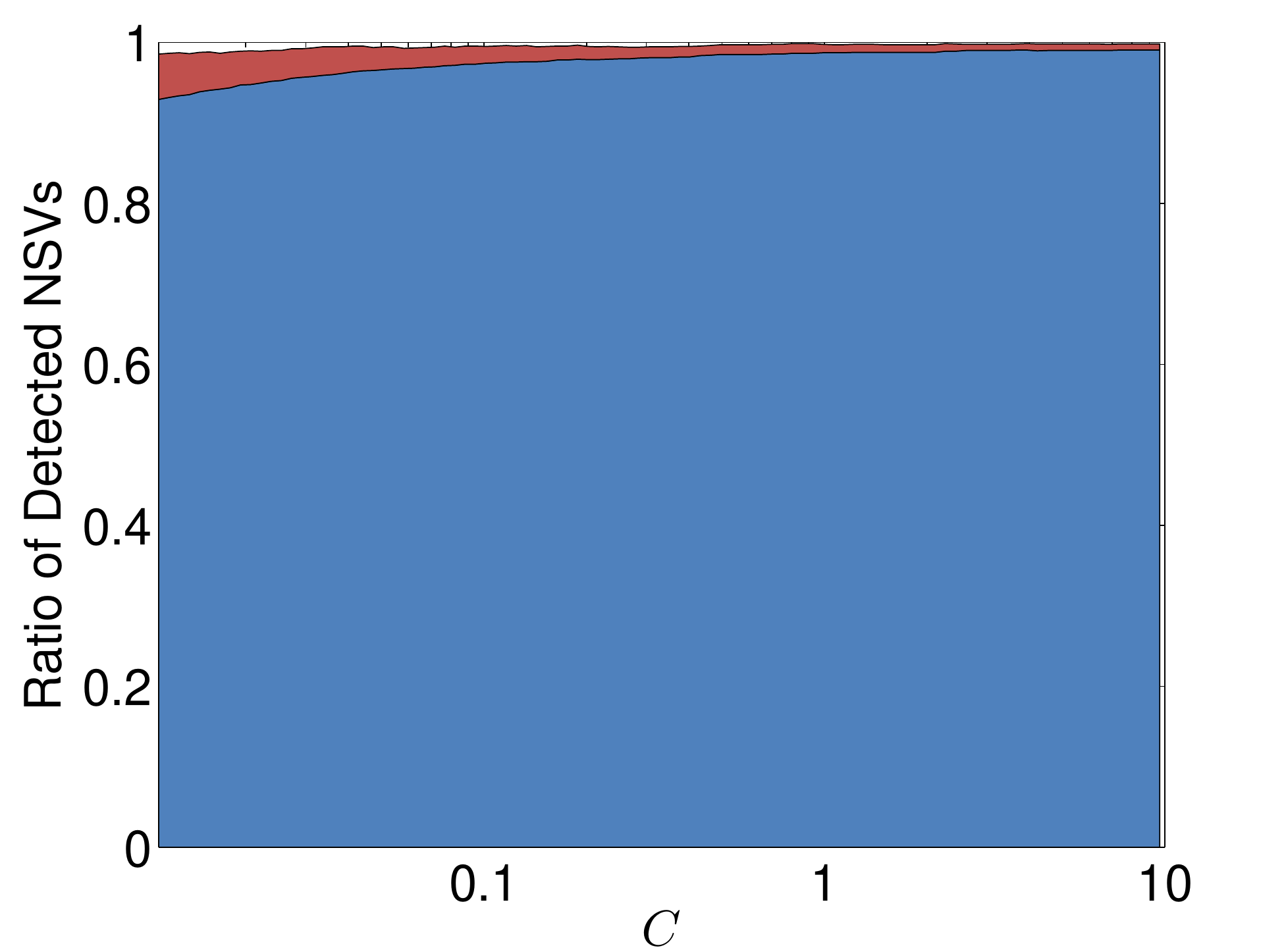}
}
\subfigure[Toy2] { \label{fig:q2e1}
\includegraphics[width=0.22\columnwidth,angle=90]{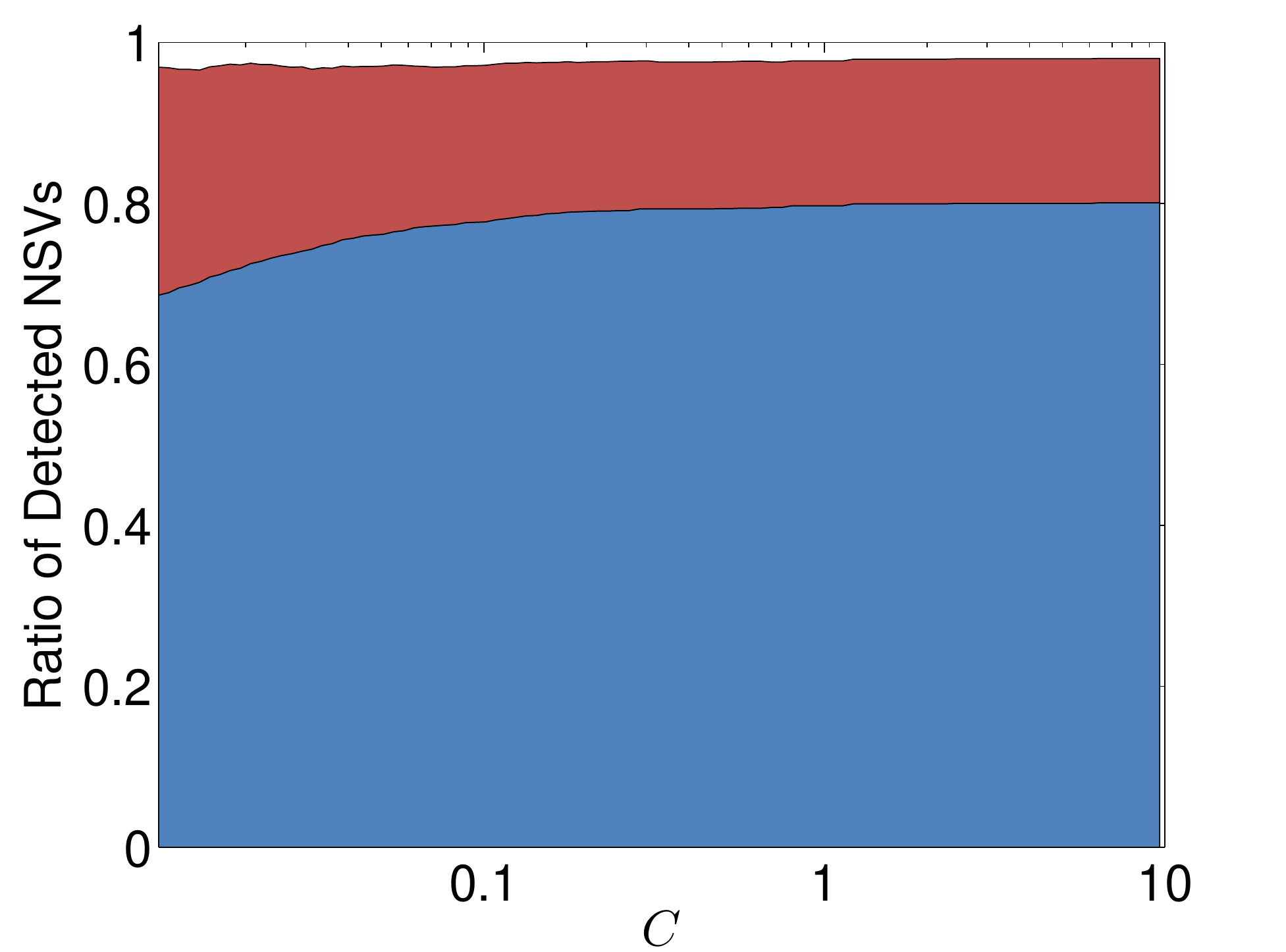}
}
\subfigure[Toy3] { \label{fig:q233e1}
\includegraphics[width=0.22\columnwidth,angle=90]{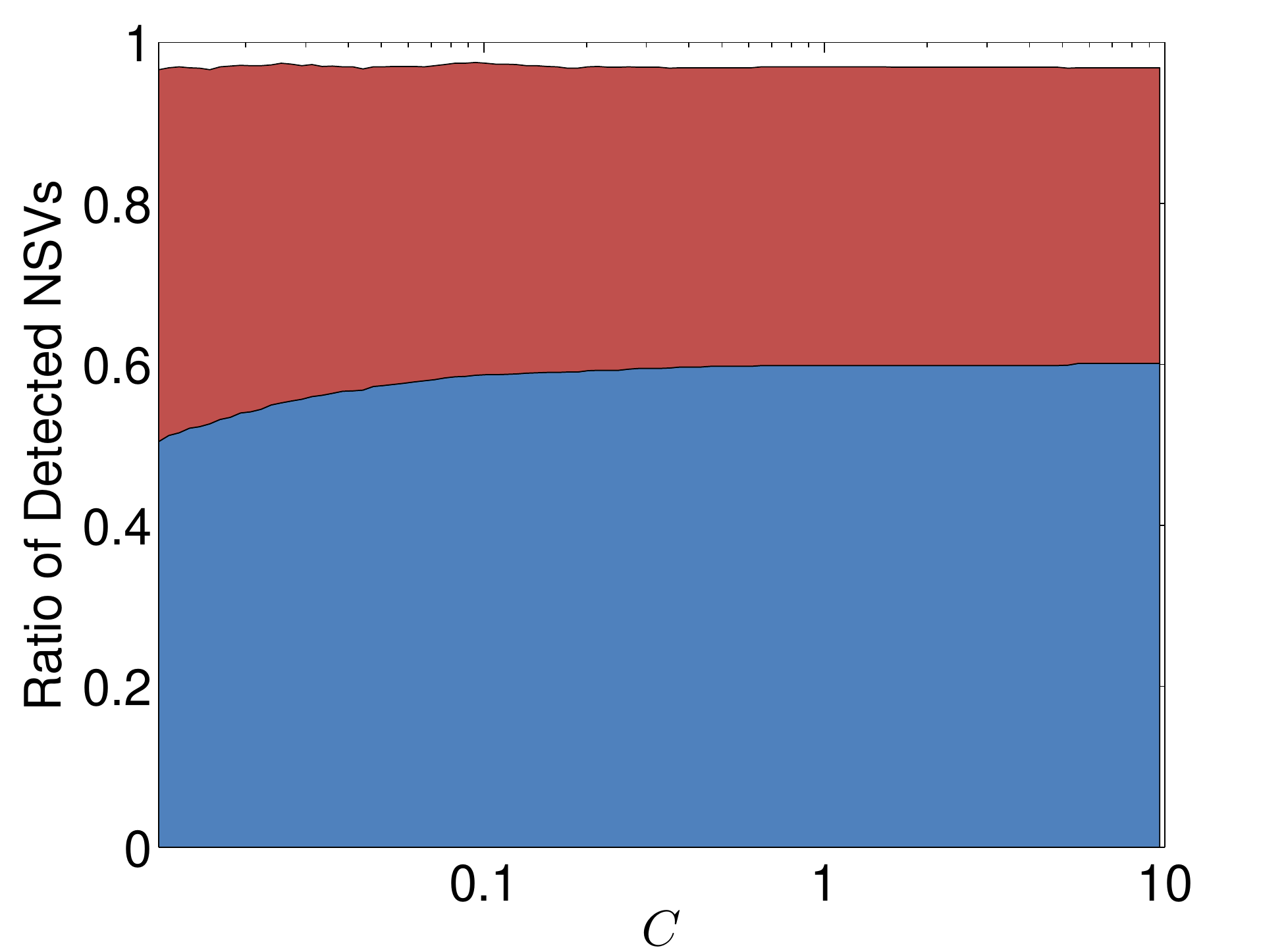}
}
}
\caption{DVI$_s$ for three 2D synthetic data sets. The first row shows the plots of the data. Cyan and magenta dotted lines are the resulting decision functions at $C=10^{-2}$ and $C=10$, respectively. The second row presents the rejection rates of DVI$_s$ with the given $100$ parameter values.}
\label{fig:rej_ratio_synthetic}
\end{figure*}

\begin{figure*}[ht]
\centering{
\subfigure[IJCNN1] {
\includegraphics[width=0.3\columnwidth]{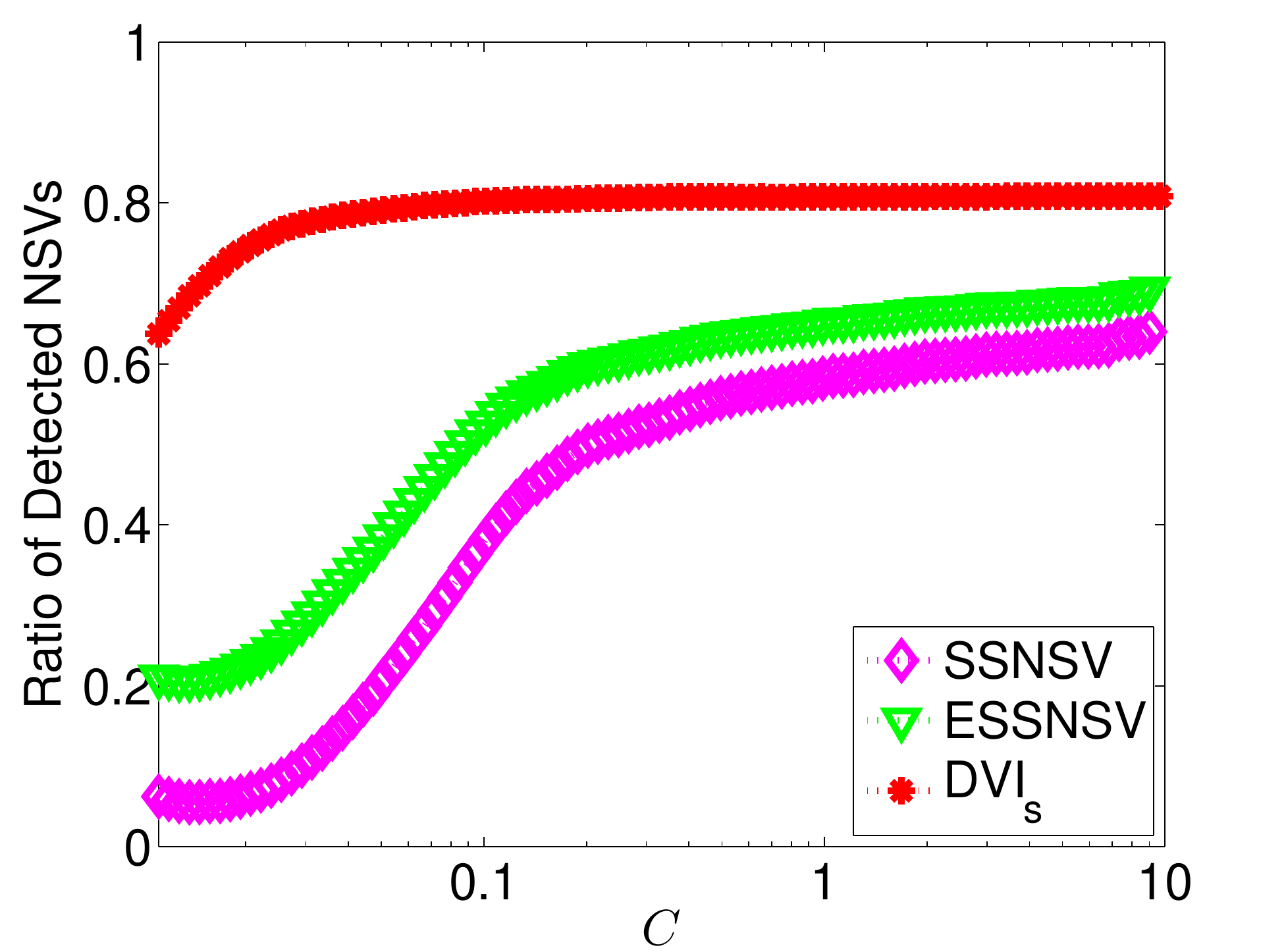}
}
\subfigure[Wine] {
\includegraphics[width=0.3\columnwidth]{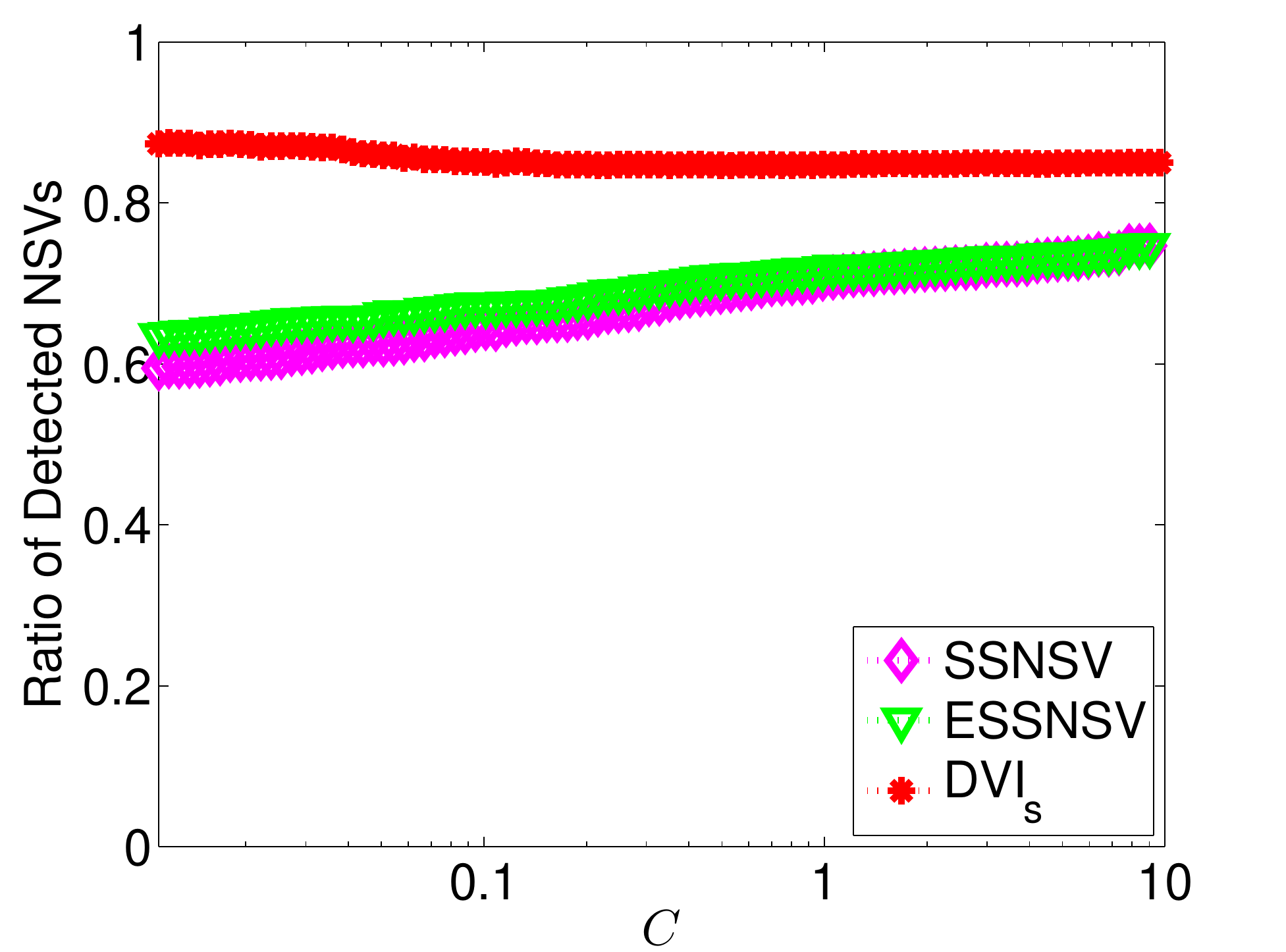}
}
\subfigure[Forest Covertype] {
\includegraphics[width=0.3\columnwidth]{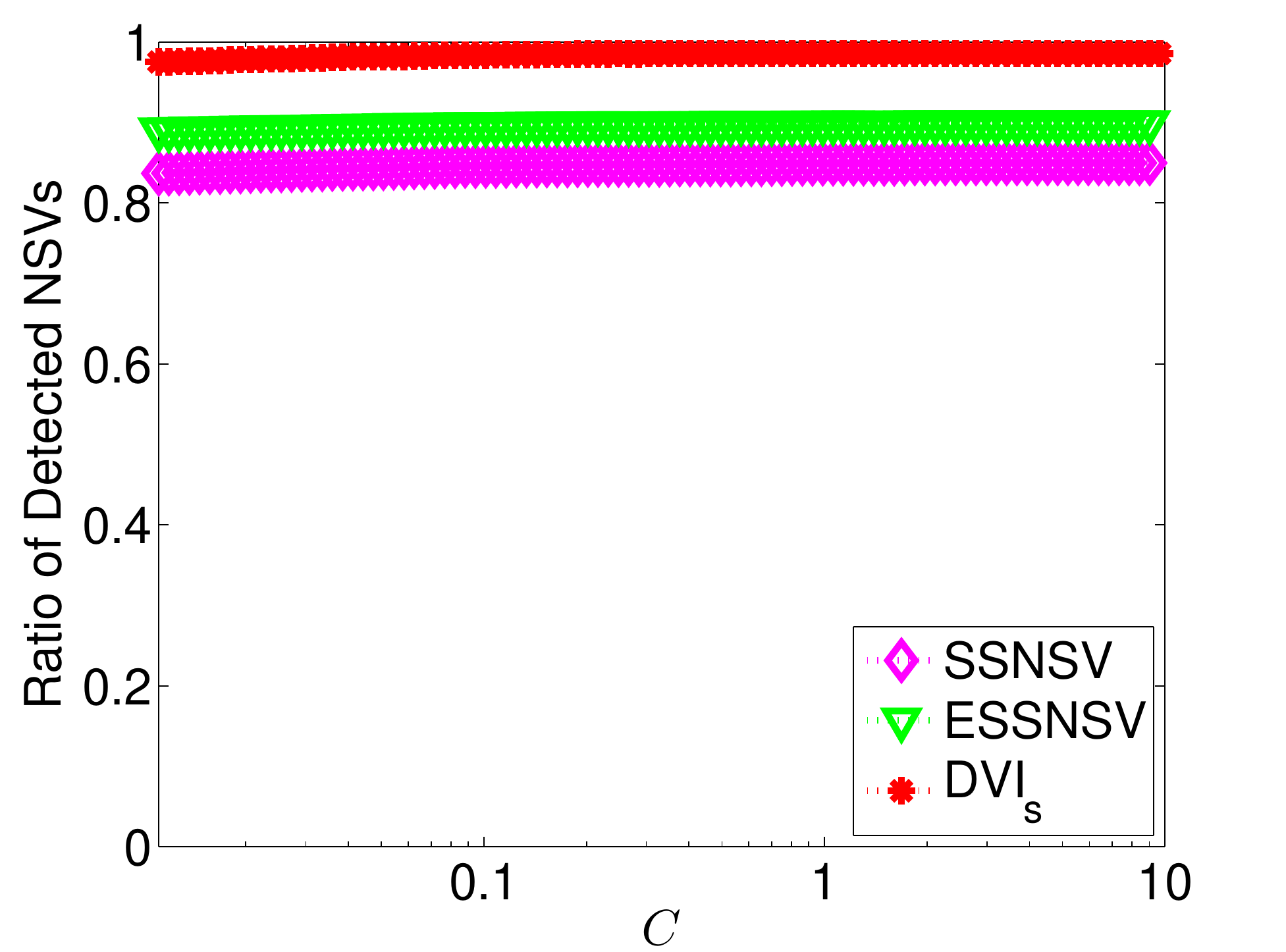}
}
}
\caption{Comparison of the performance of SSNSV, ESSNSV and DVI$s$ for SVM on three real data sets.}
\label{fig:rej_ratio_real}
\end{figure*}

\begin{figure*}[ht!]
\centering{
\subfigure[Magic Gamma Telescope] {
\includegraphics[width=0.3\columnwidth]{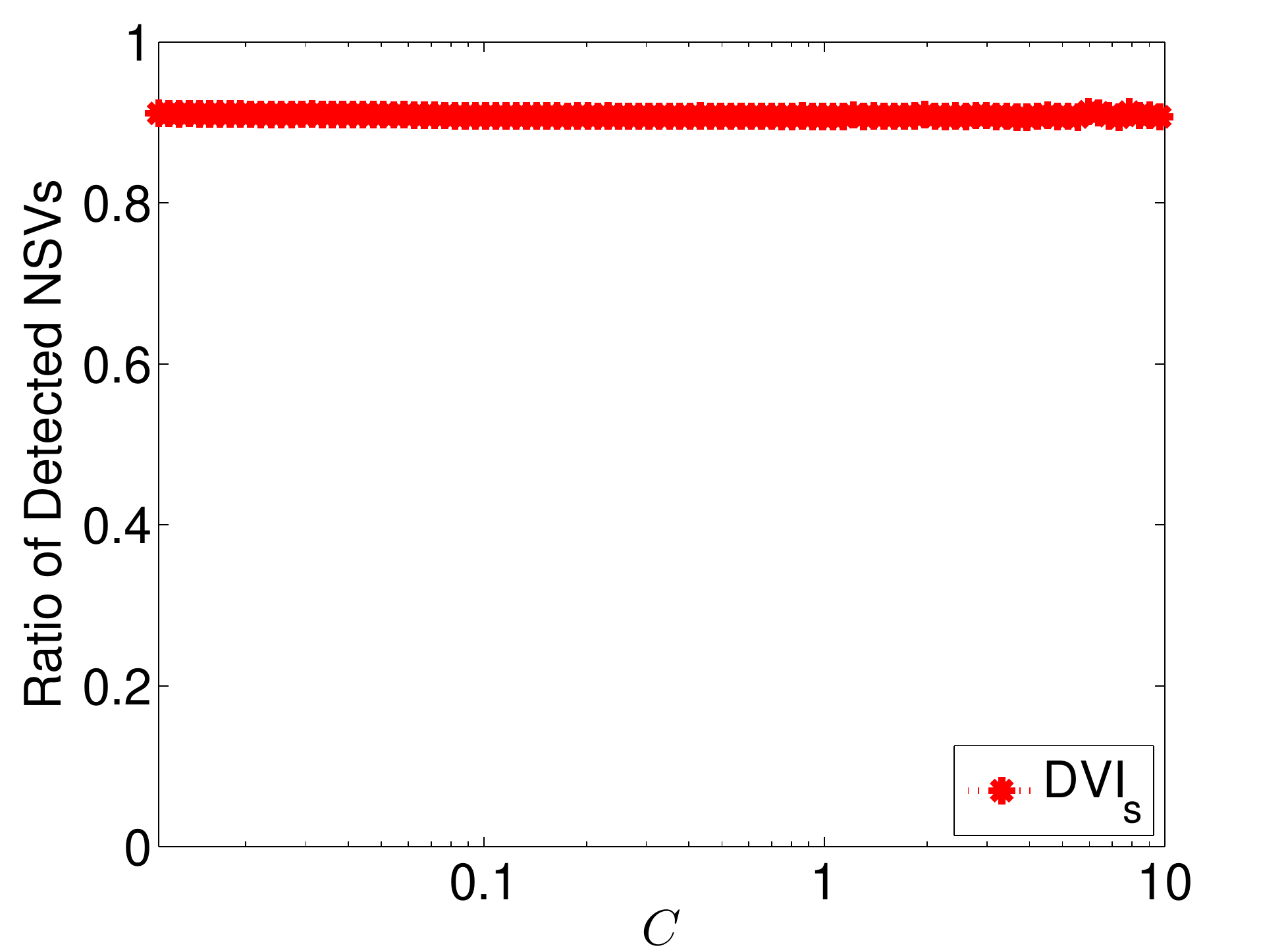}
}
\subfigure[Computer] {
\includegraphics[width=0.3\columnwidth]{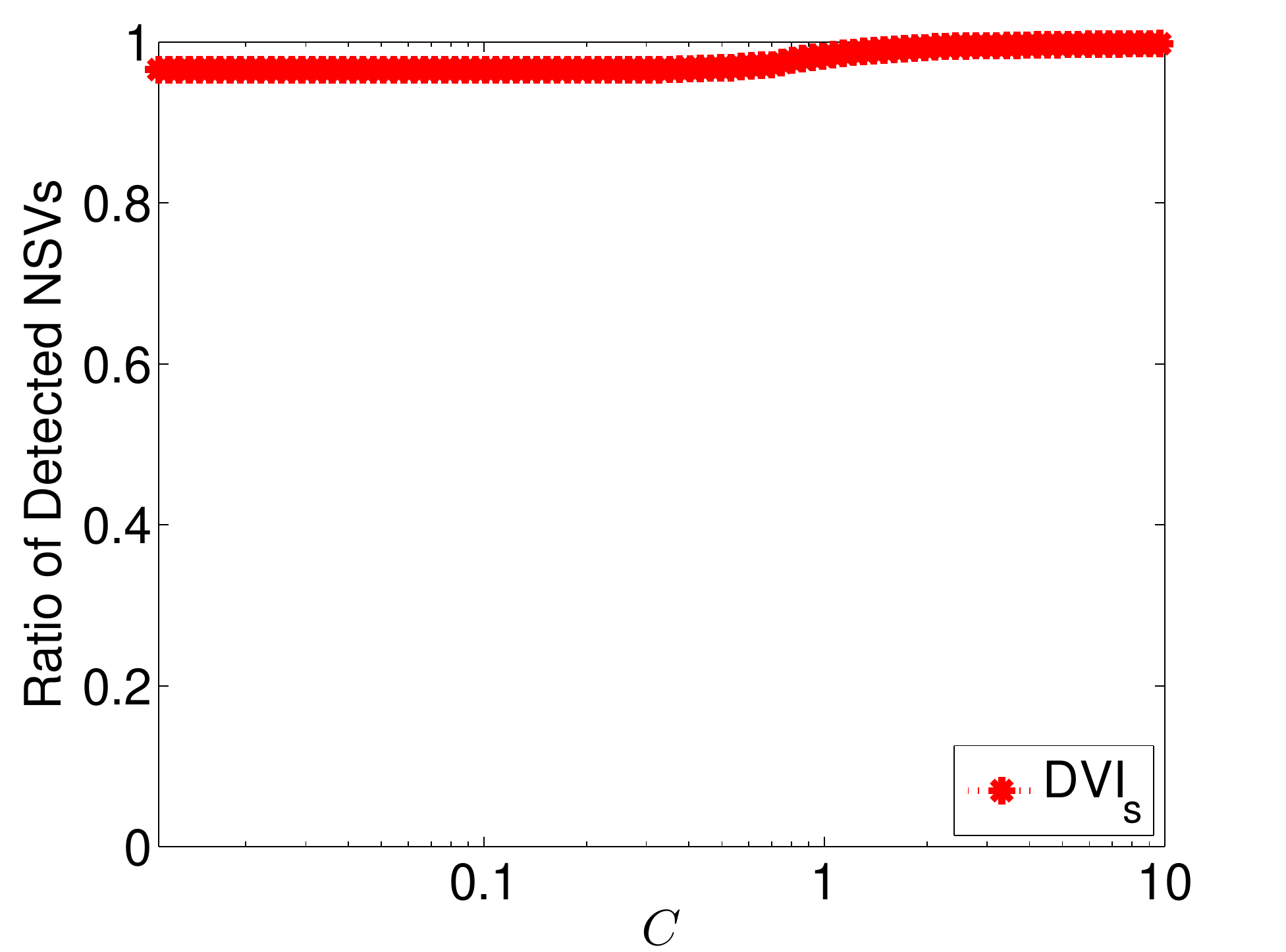}
}
\subfigure[Houses] {
\includegraphics[width=0.3\columnwidth]{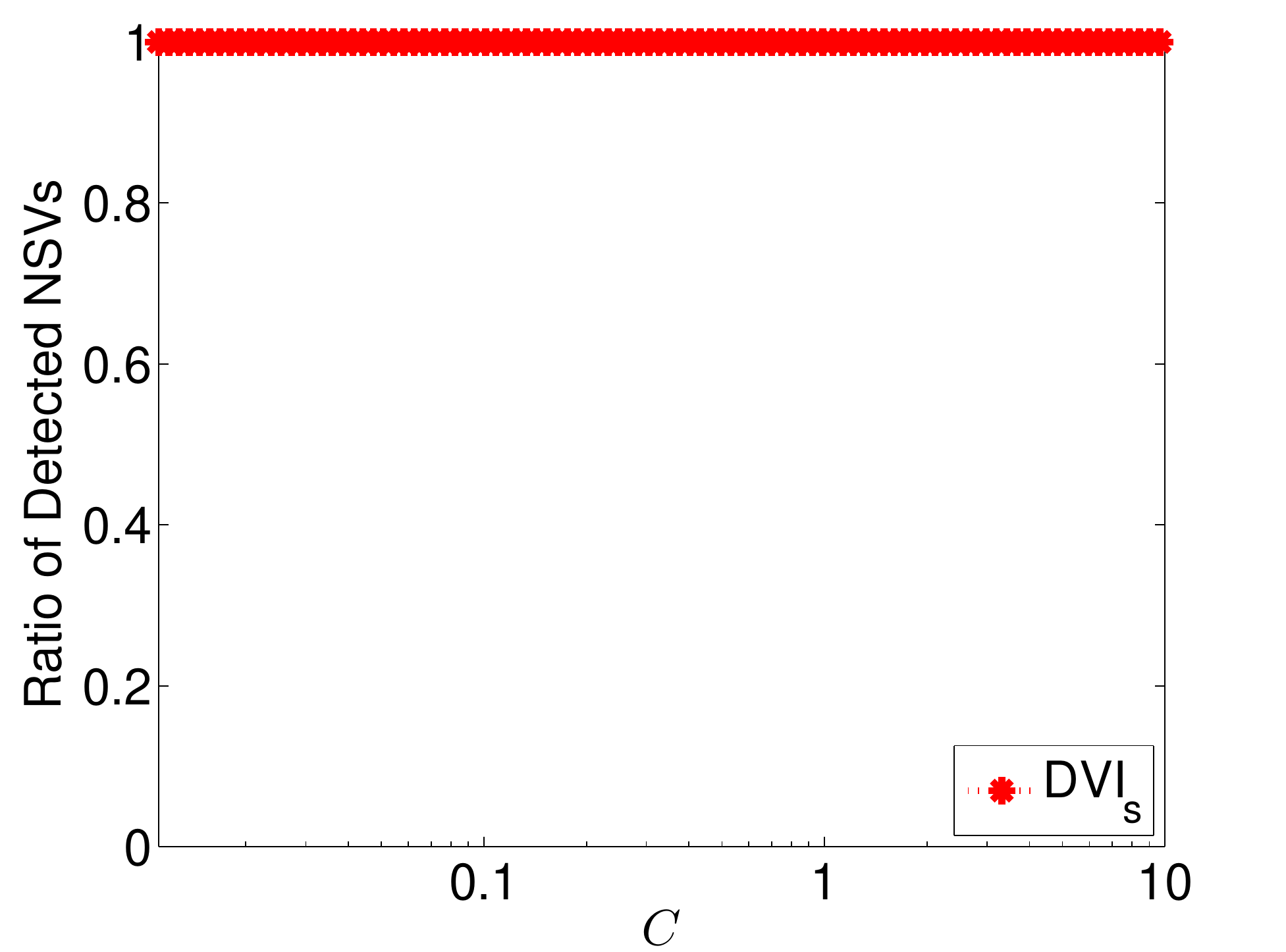}
}
}
\caption{Rejection ratio of DVI$s$ for LAD on three real data sets.}
\label{fig:rej_ratio_real_rls}
\end{figure*}

{\bf Synthetic Data Sets} In this experiment, we show that DVI$_s$ are very effective in discarding non-support vectors even for largely overlapping classes. We evaluate DVI$_s$ rules on three synthetic data sets, i.e., Toy1, Toy2 and Toy3, plotted in the first row of \figref{fig:rej_ratio_synthetic}. For each data set, we generate two classes. Each class has $1000$ data points and is generated from $N(\{\mu,\mu\}^T,0.75^2{\bf I})$, where ${\bf I}\in\Re^{2\times 2}$ is the identity matrix. For the positive classes (the red dots), $\mu=1.5, 0.75, 0.5$, for Toy1, Toy2 and Toy 3, respectively; and $\mu=-1.5, -0.75, -0.5$, for the negative classes (the blue dots). From the plots, we can observe that when $|\mu|$ decreases, the two classes increasingly overlap and thus the number of data instances belong to the set $\mathcal{L}$ increases.
\begin{table}
\vspace{-0.05in}
\caption{Running time (in seconds) for solving the SVM problems with $100$ parameter values by (a) ``Solver" (solver without screening); (b) ``Solver+DVI$_s$" (solver combined with DVI$_s$). ``DVI$_s$" is the total running time (in seconds) of the rule. ``Init." is the running time to solve SVM with the smallest parameter value.}
\label{table:time_synthetic_svm}
\begin{center}
\begin{tabular}{ l|c|c|c|c|c }
  \hline
   & Solver & Solver+DVI$_s$ & DVI$_s$ & Init. & \multicolumn{1}{c}{Speedup} \\
  \hline\hline
  Toy1  & 11.83 & 0.20 & 0.02 & 0.12 & {\bf 59.15}  \\\hline
  Toy2  & 13.68 & 0.52 & 0.03 & 0.15 & {\bf 26.31} \\\hline
  Toy3  & 15.35 & 0.61 & 0.03 & 0.16 & {\bf 25.16}  \\
  \hline
\end{tabular}
\end{center}
\end{table}

The second row of \figref{fig:rej_ratio_synthetic} presents the stacked area charts of the rejection rates. For convenience, let $\widetilde{\mathcal{R}}$ and $\widetilde{\mathcal{L}}$ be the indices of data instances which are identified by DVI$_s$ as members of $\mathcal{R}$ and $\mathcal{L}$, respectively. Then, the blue and red regions present the ratios of $|\widetilde{\mathcal{R}}|/l$ and $|\widetilde{\mathcal{L}}|/l$ (recall that, $l$ is the number of data instances, which is $2000$ for this experiment). We can see that, for Toy1, the two classes are clearly apart from each other and thus most of the data instances belong to the set $\mathcal{R}$. The first chart in the second row of \figref{fig:rej_ratio_synthetic} indicates that the proposed DVI$_s$ can identify almost all of the non-support vectors and thus the speedup is almost $60$ times compared to the solver without screening (please refer to Table \ref{table:time_synthetic_svm}). When the two classes have a large overlap, e.g., Toy3, the number of data instances in $\mathcal{L}$ significantly increases. This will generally impose great challenge for the solver. But even for this challenging case, DVI$_s$ is still able to identify a large portion of the non-support vectors as indicated by the last charts in the second row of \figref{fig:rej_ratio_synthetic}. Notice that, for Toy3, $|\widetilde{\mathcal{L}}|$ is comparable to $|\widetilde{\mathcal{R}}|$. Table \ref{table:time_synthetic_svm} shows that the speedup gained by DVI$_s$ is about $25$ times for this challenging case. It is worthwhile to mention that the running time of ``Solver$+$DVI$_s$" in Table \ref{table:time_synthetic_svm} includes the running time (the $5^{th}$ column of Table \ref{table:time_synthetic_svm}) for solving SVM with the smallest parameter value.

{\bf Real Data Sets} In this experiment, we compare the performance of SSNSV, ESSNSV and DVI$_s$ in terms of the rejection ratio, that is, the ratio between the number of data instances identified as members of $\mathcal{R}$ or $\mathcal{L}$ by the screening rules and the number of total data instances.
\begin{table}
\caption{Running time (in seconds) for solving the SVM problems along the $100$ parameter values on three real data sets. In ``Solver+SSNSV" and ``Solver+ESSNSV", ``Init." reports the running time for solving SVM at the smallest and the largest parameter values since they are required to run SSNSV and ESSNSV. In ``Solver+DVI$_s$", ``Init." reports the running time for solving SVM at the smallest parameter value which is sufficient to run DVI$_s$. The running time reported by Init. is included in the total running time of the solver equipped with the screening methods.
}
\label{table:time_real_svm}
\begin{center}
\begin{tabular}{ l|c|c|c }
  \hline
   \multicolumn{3}{c|}{IJCNN1  ($l=49990,n=22$)} & \multicolumn{1}{c}{Speedup} \\
  \hline\hline
  Solver  & Total & 4669.14 & - \\\hline\hline
  \multirow{3}{*}{Solver+SSNSV}  & SSNSV & 2.08 &\\ \cline{2-3}
                                 & Init. & 92.45 & 2.31\\ \cline{2-3}
                                 &  Total & 2018.55 &\\ \hline\hline
  \multirow{3}{*}{Solver+ESSNSV}  & ESSNSV & 2.09\\ \cline{2-3}
                                 & Init. & 91.33 & 3.01\\ \cline{2-3}
                                 &  Total & 1552.72\\ \hline\hline
  \multirow{3}{*}{Solver+DVI$_s$}  & DVI$_s$ & 0.99\\ \cline{2-3}
                                 & Init.& 42.67 & {\bf 5.64}\\ \cline{2-3}
                                 & Total & 828.02\\ \hline
  \multicolumn{3}{c}{} \\
  \hline
  \multicolumn{3}{c}{Wine  ($l=6497,n=12$)} & \multicolumn{1}{|c}{Speedup} \\
  \hline\hline
  Solver  & Total & 76.52 & - \\\hline\hline
  \multirow{3}{*}{Solver+SSNSV}  & SSNSV & 0.02\\ \cline{2-3}
                                 & Init. & 1.56 & 3.50\\ \cline{2-3}
                                 &  Total & 21.85\\ \hline\hline
  \multirow{3}{*}{Solver+ESSNSV}  & ESSNSV & 0.03\\ \cline{2-3}
                                 & Init. & 1.60 & 4.47\\ \cline{2-3}
                                 &  Total & 17.17\\ \hline\hline
  \multirow{3}{*}{Solver+DVI$_s$}  & DVI$_s$ & 0.01\\ \cline{2-3}
                                 & Init.& 0.67 & {\bf 6.59}\\ \cline{2-3}
                                 & Total & 11.62\\ \hline
  \multicolumn{3}{c}{} \\
  \hline
  \multicolumn{3}{c}{Forest Covertype  ($l=37877,n=54$)} & \multicolumn{1}{|c}{Speedup} \\
  \hline\hline
  Solver  & Total & 1675.46 & - \\\hline\hline
  \multirow{3}{*}{Solver+SSNSV}  & SSNSV & 2.73\\ \cline{2-3}
                                 & Init. & 35.52 & 7.60\\ \cline{2-3}
                                 &  Total & 220.58\\ \hline\hline
  \multirow{3}{*}{Solver+ESSNSV}  & ESSNSV & 2.89\\ \cline{2-3}
                                 & Init. & 36.13 & 10.72\\ \cline{2-3}
                                 &  Total & 156.23\\ \hline\hline
  \multirow{3}{*}{Solver+DVI$_s$}  & DVI$_s$ & 1.27\\ \cline{2-3}
                                 & Init.& 12.57 & {\bf 79.18}\\ \cline{2-3}
                                 & Total & 21.16\\ \hline
  \multicolumn{3}{c}{} \\

\end{tabular}
\end{center}
\end{table}
\figref{fig:rej_ratio_real} shows the rejection ratios of the three screening rules on three real data sets. We can observe that DVI$_s$ rules identify far more non-support vectors than SSNSV and ESSNSV. For IJCNN1, about $80\%$ of the data instances are identified as non-support vectors by DVI$_s$. Therefore, as indicated by Table \ref{table:time_real_svm} the speedup gained by DVI$_s$ is about $5$ times. For the Wine data set, more than $80\%$ of the data instances are identified to belong to $\mathcal{R}$ or $\mathcal{L}$ by DVI$_s$. As indicated in Table \ref{table:time_real_svm}, the speedup is about $6$ times gained by DVI$_s$. For the Forest Covertype data set, almost all of data instances' membership can be determined by DVI$_s$. Table \ref{table:time_real_svm} shows that the speedup gained by DVI$_s$ is almost $80$ times, which is much higher than that of SSNSV and ESSNSV. Moreover, \figref{fig:rej_ratio_real} demonstrates that ESSNSV is more effective in identifying non-support vectors than SSNSV, which is consistent with our analysis.

\subsection{DVI for LAD}\label{subsection:experiment_LAD}

\begin{table}
\caption{Running time (in seconds) for solving the LAD problems with the given $100$ parameter values on three real data sets. In ``Solver+DVI$_s$", ``Init." reports the running time for solving LAD at the smallest parameter value which is required to run DVI$_s$. Init. is included in the total running time of Solver$+$DVI$_s$.
}\label{table:time_real_rls}
\begin{center}
\begin{tabular}{ l|c|c|c }
  \hline
   \multicolumn{3}{c|}{Magic Gamma Telescope ($l=19020,n=10$)} & \multicolumn{1}{c}{Speedup} \\
  \hline\hline
  Solver  & Total & 122.34 & - \\\hline\hline
  \multirow{3}{*}{Solver+DVI$_s$}  & DVI$_s$ & 0.28\\ \cline{2-3}
                                 & Init.& 0.12 & {\bf 9.86}\\ \cline{2-3}
                                 & Total & 12.41\\ \hline
  \multicolumn{3}{c}{} \\
  \hline
  \multicolumn{3}{c}{Computer  ($l=8192,n=12$)} & \multicolumn{1}{|c}{Speedup}\\
  \hline\hline
  Solver  & Total & 5.38 & - \\\hline\hline
  \multirow{3}{*}{Solver+DVI$_s$}  & DVI$_s$ & 0.08\\ \cline{2-3}
                                 & Init.& 0.05 & {\bf 19.21}\\ \cline{2-3}
                                 & Total & 0.28\\ \hline

  \multicolumn{3}{c}{} \\
  \hline
  \multicolumn{3}{c}{Houses ($l=20640,n=8$)} & \multicolumn{1}{|c}{Speedup}\\
  \hline\hline
  Solver  & Total & 21.43 & - \\\hline\hline
  \multirow{3}{*}{Solver+DVI$_s$}  & DVI$_s$ & 0.06\\ \cline{2-3}
                                 & Init.& 0.10 & {\bf 114.91}\\ \cline{2-3}
                                 & Total & 0.19\\ \hline
  \multicolumn{3}{c}{} \\

\end{tabular}
\end{center}
\end{table}

In this experiment, we evaluate the performance of DVI$_s$ for LAD on three real data sets: (a) Magic Gamma Telescope data set \cite{Bache+Lichman:2013}; (b) Computer data set \cite{Rasmussen}; (c) Houses data set \cite{Pace1997}. \figref{fig:rej_ratio_real_rls} shows the rejection ratio of DVI$_s$ rules for the three data sets. We can observe that the rejection ratio of DVI$_s$ on Magic Gamma Telescope data set is about $90\%$, leading to a $10$ times speedup as indicated in Table \ref{table:time_real_rls}. For the Computer and Houses data sets, we can see that the rejection rates are very close to $100\%$, i.e., almost all of the data instances' membership can be determined by the DVI$_s$ rules. As expected, Table \ref{table:time_real_rls} shows that the resulting speedup are about $20$ and $115$ times, respectively. Notice that, the speedup for the Houses data set is more than two orders of magnitude. These results demonstrate the effectiveness of the proposed DVI rules.

\section{Conclusion}\label{section:conclusion}

In this paper, we develop new screening rules for a class of supervised learning problems by studying their dual formulation with the variational inequalities. Our framework includes two well known models, i.e., SVM and LAD, as special cases. The proposed DVI rules are very effective in identifying non-support vectors for both SVM and LAD, and thus result in substantial savings in the computational cost and memory. Extensive experiments on both synthetic and real data sets demonstrate the effectiveness of the proposed DVI rules. We plan to extend the framework of DVI to other supervised learning problems, e.g., weighted SVM \cite{Yang2005}, RWLS (robust weighted least squres) \cite{Chatterjee1997}, robust PCA \cite{Ding2006}, robust matrix factorization \cite{Ke2005}.


\clearpage
\newpage
\nocite{langley00}
\bibliography{refs}
\bibliographystyle{plain}

\clearpage
\newpage
\onecolumn
\appendix




\section{Proof of Lemma \ref{lemma:biconjugate}}

Before we prove Lemma \ref{lemma:biconjugate}, let us cite the following technical lemma.

\begin{lemma}\label{lemma:biconjugate_iff}
\cite{Hiriart-Urruty1993} The function $f$ is equal to its biconjugate $f^{**}$ if and only if $f\in\Gamma_0(\Re^n)$.
\end{lemma}

We are now ready to derive a simple proof of Lemma \ref{lemma:biconjugate} based on Lemma \ref{lemma:biconjugate_iff}.

\begin{proof}
In order to show $\varphi^{**}=\varphi$, it is enough to show $\varphi\in\Gamma_0(\Re)$ according to Lemma \ref{lemma:biconjugate_iff}. Therefore we only to check the following three conditions:

1). Properness: because $\varphi:\Re\rightarrow\Re_+$, i.e., there exists $t\in\Re$ such that $\varphi(t)$ is finite, $\varphi$ is proper.

2). Lower semi-continuality: $\varphi$ is lower semicontinuous because it is continuous.

3). Convexity: the convexity of $\varphi$ is due to the its sublinearity, see Definition \ref{def:sublinear}.

Thus, we have $\varphi\in\Gamma_0(\Re)$, which completes the proof.
\end{proof}

\section{Proof of Lemma \ref{lemma:biconjugate_indicator}}

To prove Lemma \ref{lemma:biconjugate_indicator}, we need to following results.

\begin{lemma}\label{lemma:indicator_sublinear}
\cite{Ruszczynski2006} Let $Z\subseteq\Re^n$ be a convex and closed set. Let us define the support function of $Z$ as
\begin{align}
\sigma_{Z}(s):=\sup_{{\bf x}\in Z}{\bf s}^T{\bf x},
\end{align}
and the indicator function $\iota_{Z}$ as
\begin{align}
\iota_{Z}({\bf x})=
\begin{cases}
0,\hspace{4mm}{\rm if}\hspace{2mm}{\bf x}\in Z,\\
\infty,\hspace{2mm}{\rm otherwise}.
\end{cases}
\end{align}
Then
\begin{align}
\sigma_{Z}^*=\iota_{Z},\hspace{3mm}{\rm amd}\hspace{1mm}\iota_{Z}^*=\sigma_{Z}.
\end{align}
\end{lemma}

\begin{theorem}\label{theorem:sublinear_support}
\cite{Hiriart-Urruty1993} Let $\sigma\in\Gamma_0(\Re^n)$ be a sublinear function, then $\sigma$ is the support function of the nonempty closed convex set
\begin{align}
S_{\sigma}:=\{{\bf s}\in\Re^n:{\bf s}^T{\bf d}\leq\sigma({\bf d}),\,\,\forall {\bf d}\in\Re^n\}.
\end{align}
\end{theorem}

We are now ready to prove Lemma \ref{lemma:biconjugate_indicator}.

\begin{proof}
Due to Lemma \ref{lemma:indicator_sublinear} and Theorem \ref{theorem:sublinear_support}, we can see that, there is a nonempty closed convex set $Z\subseteq\Re$ such that
\begin{align}
\varphi(t)=\sup_{s\in Z}st,\,\,\forall t\in\Re,
\end{align}
where
\begin{align}\label{eqn:interval_bound}
Z:=\{s:st\leq\varphi(t),\,\,\forall t\in\Re\}.
\end{align}

Let $t=1$ and $-1$ respectively, \eqref{eqn:interval_bound} implies that
\begin{align}
\sup_{s\in Z} s\leq \varphi(1)\hspace{3mm} {\rm and}\hspace{3mm}\inf_{s\in Z} s\geq \varphi(-1).
\end{align}
Therefore, $Z$ is a closed and bounded interval, i.e., $Z=[\alpha,\beta]$ with $\alpha,\beta\in\Re$.

Next, let us show that $\alpha\neq\beta$. In fact, in view of the nonnegativity of $\varphi$ and \eqref{eqn:interval_bound}, it is easy to see that $0\in Z$. Therefore, if $\alpha=\beta$, we must have $Z=\{0\}$. Thus, Lemma \ref{lemma:indicator_sublinear} implies that
\begin{align}
\varphi=\iota^*_{Z}\equiv0,
\end{align}
which contradicts the fact that $\varphi$ is a nonconstant function. Hence, we can conclude that $\alpha<\beta$, which completes the proof.
\end{proof}

\section{Derivation of the KKT Condition in \eqref{eqn:KKT_general}}

The problem in (\ref{prob:general_dual1}) can be written as follows:
\begin{align}
\min_{\theta}\,\,&\frac{C}{2}\|{\bf Z}^T\theta\|^2-\langle\bar{\bf y},\theta\rangle,\\ \nonumber
{\rm s.t.}\,\,&\theta_i\in[\alpha,\beta],\,\,i=1,\ldots,l.
\end{align}

Therefore, we can see that the Lagrangian is
\begin{align}
L(\theta,\mu,\nu) = \frac{C}{2}\|{\bf Z}^T\theta\|^2-\langle\bar{\bf y},\theta\rangle+ \sum_{i=1}^l\mu_i(\alpha-\theta_i)+\sum_{i=1}^l\nu_i(\theta_i-\beta),
\end{align}
where $\mu=(\mu_1,\ldots,\mu_l)^T$, $\nu=(\nu_1,\ldots,\nu_l)^T$, and $\mu_i\geq0$, $\nu_i\geq0$ for all $i=1,\ldots,l$. $\mu$ and $\nu$ are in fact the vector of Lagrangian multipliers.

For simplicity, let us denote $\theta^*(C)$ by $\theta^*$. Then the KKT conditions \cite{Boyd04} are
\begin{align}\label{eqn:KKT_derivation}
\frac{\partial L(\theta,\mu,\nu)}{\partial \theta}|_{\theta^*}=0\Rightarrow C{\bf Z}{\bf Z}^T\theta^*-\bar{\bf y}-\mu+\nu=0,
\end{align}
\begin{align}\label{eqn:slackness}
\begin{array}{lcl}
\mu_i(\alpha-\theta_i^*)&=&0,\\
\nu_i(\theta^*_i-\beta)&=&0,
\end{array}
i=1,\ldots,l.
\end{align}
\eqref{eqn:slackness} is known as the complementary slackness condition. The equation in (\ref{eqn:KKT_derivation}) actually involves $l$ equations. We can write down the $i^{th}$ equation as follows:
\begin{align}\label{eqn:KKT_i}
C\langle{\bf Z}^T\theta^*,a_i{\bf x}_i\rangle-\mu_i+\nu_i=b_iy_i.
\end{align}
Recall that the $i^{th}$ column of ${\bf Z}$ is $a_i{\bf x}_i$. In view of \eqref{eqn:slackness} and \eqref{eqn:KKT_i}, we can see that:

1. if $\theta_i^*=\alpha$, then $\nu_i=0$ and \eqref{eqn:KKT_i} results in
\begin{align}\label{ineqn:KKT_i1}
C\langle{\bf Z}^T\theta^*,a_i{\bf x}_i\rangle \geq b_iy_i;
\end{align}

2.  if $\theta_i^*\in(\alpha,\beta)$, then $\mu_i=\nu_i=0$ and \eqref{eqn:KKT_i} results in
\begin{align}\label{ineqn:KKT_i2}
C\langle{\bf Z}^T\theta^*,a_i{\bf x}_i\rangle = b_iy_i;
\end{align}

3. if $\theta_i^*=\beta$, then $\mu_i=0$ and \eqref{eqn:KKT_i} results in
\begin{align}\label{ineqn:KKT_i3}
C\langle{\bf Z}^T\theta^*,a_i{\bf x}_i\rangle \leq b_iy_i.
\end{align}

Then, in view of the inequalities in (\ref{ineqn:KKT_i1}), (\ref{ineqn:KKT_i2}) and (\ref{ineqn:KKT_i3}), and \eqref{eqn:primal_dual}, it is straightforward to derive the KKT condition in (\ref{eqn:KKT_general}).

\section{Proof of Lemma \ref{lemma:dual_reduce}}

\begin{proof}
The first part of the statement is trivial by the definition of $\hat{\mathcal{R}}$ and $\hat{\mathcal{L}}$. Therefore, we only consider the second part of the statement.

Let ${\bf G}={\bf Z}{\bf Z}^T$.
By permuting the columns and rows of ${\bf G}$, we have
\begin{align*}
\hat{{\bf G}}=
\begin{pmatrix}
\hat{{\bf G}}_{11}&\hat{{\bf G}}_{12}\\
\hat{{\bf G}}_{21}&\hat{{\bf G}}_{22}\\
\end{pmatrix}
=
\begin{pmatrix}\vspace{1mm}
[{\bf X}^T]_{\hat{\mathcal{S}}^{\rm c}}^T[{\bf X}^T]_{\hat{\mathcal{S}}^{\rm c}} & [{\bf X}^T]_{\hat{\mathcal{S}}^{\rm c}}^T[{\bf X}^T]_{\hat{\mathcal{S}}}\\
[{\bf X}^T]_{\hat{\mathcal{S}}}^T[{\bf X}^T]_{\hat{\mathcal{S}}^{\rm c}} & [{\bf X}^T]_{\hat{\mathcal{S}}}^T[{\bf X}^T]_{\hat{\mathcal{S}}}\\
\end{pmatrix}.
\end{align*}
As a result, the objective function of problem (\ref{prob:general_dual1}) can be rewritten as
\begin{align}\label{eqn:dual_rls_reduced}
\frac{C}{2}[\theta]_{\hat{\mathcal{S}}^{\rm c}}^T\hat{{\bf G}}_{11}[\theta]_{\hat{\mathcal{S}}^{\rm c}}-\hat{\bf y}^T[\theta]_{\hat{\mathcal{S}}^{\rm c}}+R([\theta]_{\hat{\mathcal{S}}})
\end{align}
where
\begin{align}
\hat{\bf y}={\bf y}_{\hat{\mathcal{S}}^{\rm c}}-C\hat{\bf G}_{12}[\theta]_{\hat{\mathcal{S}}},\hspace{1mm},
\end{align}
\begin{align}
R([\theta]_{\hat{\mathcal{S}}})=\frac{C}{2}
[\theta]_{\hat{\mathcal{S}}}^T\hat{{\bf G}}_{22}[\theta]_{\hat{\mathcal{S}}}-{\bf y}_{\hat{\mathcal{S}}}^T[\theta]_{\hat{\mathcal{S}}}
\end{align}
Due to the assumption that $[\theta^*(C)]_{\hat{\mathcal{S}}}$ is known, $\hat{\bf y}$ and $R([\theta]_{\hat{\mathcal{S}}})$ can be treated as constants, and thus problem (\ref{prob:general_dual1})
reduces to problem (\ref{prob:dual_reduced}).
\end{proof}

\section{Improving SSNSV via VI}

In this section, we describe how to strictly improve SSNSV by using the same technique used in DVI rules in a detailed manner.

%

{\bf Estimation of ${\bf w}^*$ via VI}

We show that $\Omega_{[s_b,s_a]}$ in \eqref{eqn:jan_bound} can be strictly improved by the variational inequalities.
Consider $\mathcal{F}_{s_a}$. Because $s_a>s_b$, we can see that ${\bf w}^*(s_b)\in\mathcal{F}_{s_a}$. Therefore, by Theorem \ref{thm:vi}, we have
\begin{align}\label{ineqn:jan_vi1}
\langle{\bf w}^*(s_a),{\bf w}^*(s)-{\bf w}^*(s_a)\rangle\geq0,
\end{align}
which is the first constraint in (\ref{eqn:jan_bound}). Similarly, consider $\mathcal{F}_{s_b}$. Since $\hat{\bf w}(s_b)\in\mathcal{F}_{s_b}$, Theorem \ref{thm:vi} implies that
\begin{align*}
\langle{\bf w}^*(s),\hat{\bf w}(s_b)-{\bf w}^*(s)\rangle\geq0,
\end{align*}
which is equivalent to
\begin{align}\label{ineqn:jan_vi2}
\|{\bf w}^*(s)-\tfrac{1}{2}\hat{\bf w}(s_b)\|\leq\tfrac{1}{2}\|\hat{\bf w}(s_b)\|.
\end{align}
Clearly, the radius determined by the inequality (\ref{ineqn:jan_vi2}) is only a half of the radius determined by the second constraint in (\ref{eqn:jan_bound}).
In view of the inequalities in (\ref{ineqn:jan_vi1}) and (\ref{ineqn:jan_vi2}), we can see that ${\bf w}^*(s)$
can be bounded inside the following region:
\begin{align*}
\Omega'_{[s_b,s_a]}:=\left\{{\bf w}:
\begin{array}{l}
\langle{\bf w}^*(s_a),{\bf w}-{\bf w}^*(s_a)\rangle\geq0,\\
\|{\bf w}-\frac{1}{2}\hat{\bf w}(s_b)\|\leq\frac{1}{2}\|\hat{\bf w}(s_b)\|
\end{array}
\right\}
\end{align*}
It is easy to see that $\Omega'_{[s_b,s_a]}\subset\Omega_{[s_b,s_a]}$. As a result, the bounds in (\ref{rule1'}) and (\ref{rule2'}) with $\Omega'_{[s_b,s_a]}$ are tighter than that of $\Omega_{[s_b,s_a]}$. Thus, SSNSV \cite{Ogawa2013} can be strictly improved by the estimation in (\ref{eqn:jan_bound_vi}). In fact, we have the following theorem:

\begin{theorem}\label{thm:snsv_i}
Suppose we are given two parameters $s_a>s_b>0$, and let ${\bf w}^*(s_a)$ and $\hat{\bf w}(s_b)$ be the optimal solution at $s=s_a$ and a feasible solution at $s=s_b$, respectively. Moreover, let us define \begin{align*}
\rho &= -\|{\bf w}^*(s_a)\|^2+\tfrac{1}{2}\langle{\bf w}^*(s_a),\hat{\bf w}(s_b)\rangle\\
{\bf v}^{\perp}&={\bf v}-\tfrac{{\bf v}^T{\bf w}^*(s_a)}{\|{\bf w}^*(s_a)\|^2}{\bf w}^*(s_a), \forall {\bf v}\in\Re^n.
\end{align*}
Then, for all $s\in[s_b,s_a]$,
\begin{align}\label{snsv:rule1}
\langle{\bf w}^*(s_a),\bar{x}_i\rangle>\tfrac{2\|\bar{\bf x}_i\|}{\|\hat{\bf w}(s_b)\|}\rho\,\,{\rm and}\,\,\ell_i>1\Rightarrow i\in\mathcal{R}\Leftrightarrow \alpha_i=0,
\end{align}
where
\begin{align}\label{snsv:l}
\ell_i=-\tfrac{\langle{\bf w}^*(s_a),\bar{\bf x}_i\rangle}{\|{\bf w}^*(s_a)\|^2}\rho+\tfrac{1}{2}\langle\hat{\bf w}(s_b),\bar{\bf x}_i\rangle-\|\bar{\bf x}_i^{\perp}\|\sqrt{\frac{1}{4}\|\hat{\bf w}(s_b)\|^2-\tfrac{\rho^2}{\|{\bf w}^*(s_a)\|^2}}.
\end{align}
Similarly,
\begin{align}\label{snsv:rule2}
u_i<1\Rightarrow i\in\mathcal{L}\Leftrightarrow \alpha_i=c,
\end{align}
where
\begin{align}\label{snsv:u}
\hspace{-5mm}u_i=
\begin{cases}
\tfrac{1}{2}\left(\langle\hat{\bf w}(s_b),\bar{\bf x}_i\rangle+\|\hat{\bf w}(s_b)\|\|\bar{\bf x}_i\|\right),\\
\hspace{20mm}{\rm if}\,\,\langle{\bf w}^*(s_a),\bar{x}_i\rangle\geq-\tfrac{2\|\bar{\bf x}_i\|}{\|\hat{\bf w}(s_b)\|}\rho\\
-\tfrac{\langle{\bf w}^*(s_a),\bar{\bf x}_i\rangle}{\|{\bf w}^*(s_a)\|^2}\rho+\tfrac{1}{2}\langle\hat{\bf w}(s_b),\bar{\bf x}_i\rangle\\
\hspace{10mm}+\|\bar{\bf x}_i^{\perp}\|\sqrt{\frac{1}{4}\|\hat{\bf w}(s_b)\|^2-\tfrac{\rho^2}{\|{\bf w}^*(s_a)\|^2}},\\
\hspace{20mm}{\rm otherwise.}
\end{cases}
\end{align}
\end{theorem}
For convenience, we call the screening rule presented in Theorem \ref{thm:snsv_i} as the ``enhanced" SSNSV (ESSNSV).

To prove Theorem \ref{thm:snsv_i}, we first establish the following technical lemma.
\begin{lemma}\label{lemma:dome}
Consider the problem as follows:
\begin{align}\label{prob:dome}
\min_{\bf w} f({\bf w})={\bf v}^T{\bf w},\hspace{1mm}{\rm s.t.}\hspace{1mm}{\bf u}^T{\bf w}\leq d, \|{\bf w}-{\bf o}\|\leq r,
\end{align}
where $r>0$.
Let $d'=d-{\bf u}^T{\bf o}$ and the optimal solution of problem (\ref{prob:dome}) be $f^*$. Then we have
\begin{enumerate}
\item If ${\bf v}^T{\bf u}+\frac{\|{\bf v}\|d'}{r}\geq 0$, then
\begin{align*}
f^* = {\bf v}^T{\bf o}-r\|{\bf v}\|.
\end{align*}
\item Otherwise,
\begin{align*}
f^*={\bf v}^T{\bf o}-\|{\bf v}^{\perp}\|\sqrt{r^2-\frac{(d')^2}{\|{\bf u}\|^2}}+\frac{{\bf v}^T{\bf u}d'}{\|{\bf u}\|^2},
\end{align*}
where ${\bf v}^{\perp}={\bf v}-\frac{{\bf v}^T{\bf u}}{\|{\bf u}\|^2}{\bf u}$.
\end{enumerate}
Notice that, we assume problem (\ref{prob:dome}) is feasible, i.e., $\frac{|{\bf u}^T{\bf o}-d|}{\|{\bf u}\|}\leq r$.
\end{lemma}
\begin{proof}
 Let ${\bf z}={\bf w}-{\bf o}$, problem (\ref{prob:dome}) can be rewritten as:
\begin{align}\label{prob:dome_cv}
\min_{\bf z}{\bf v}^T{\bf z}+{\bf v}^T{\bf o},\hspace{1mm}{\rm s.t.}\hspace{1mm}{\bf u}^T{\bf z}\leq d-{\bf u}^T{\bf o}, \|{\bf z}\|\leq r.
\end{align}
Problem (\ref{prob:dome_cv}) reduces to
\begin{align}\label{prob:dome_base}
\min_{\bf z}{\bf v}^T{\bf z},\hspace{1mm}{\rm s.t.}\hspace{1mm}{\bf u}^T{\bf z}\leq d', \|{\bf z}\|\leq r.
\end{align}
To solve problem (\ref{prob:dome_base}), we make use of the Lagrangian multiplier method. For notational convenience, let $\mathcal{F}:=\{{\bf z}:{\bf u}^T{\bf z}\leq d',\|{\bf z}\|\leq r\}$.
\begin{align}\label{eqn:lagrangian}
\min_{{\bf z}\in\mathcal{F}}{\bf v}^T{\bf z}&=\min_{\bf z}\max_{\substack{\mu\geq0,\\ \nonumber
\nu\geq0}}{\bf v}^T{\bf z}+\nu({\bf u}^T{\bf z}-d')+\frac{\mu}{2}(\|{\bf z}\|^2-r^2)\\ \nonumber
&=\max_{\substack{\mu\geq0,\\ \nu\geq0}}\min_{\bf z}{\bf v}^T{\bf z}+\nu({\bf u}^T{\bf z}-d')+\frac{\mu}{2}(\|{\bf z}\|^2-r^2)\\
&=\max_{\substack{\mu\geq0,\\ \nu\geq0}}-\frac{1}{2\mu}\|{\bf v}+\nu{\bf u}\|^2-\nu d'-\frac{\mu r^2}{2}.
\end{align}

Notice that, in \eqref{eqn:lagrangian}, we make the assumption that $\mu>0$. However, we can not simply exclude this possibility. In fact, if $\mu=0$, we must have
\begin{align}\label{eqn:colinear}
{\bf v}+\nu{\bf u}=0,
\end{align}
since otherwise the function value of
$$
{\bf v}^T{\bf z}+\nu({\bf u}^T{\bf z}-d')
$$
in the second line of \eqref{eqn:lagrangian} can be made arbitrarily small. As a result, we will have
$$
g(\mu,\nu)=-\infty,
$$
which contradicts the strong duality of problem (\ref{prob:dome}) \cite{Boyd04}. [Problem (\ref{prob:dome}) is clearly lower bounded since the feasible set is compact.] Therefore, in view of \eqref{eqn:colinear}, we can conclude that $\mu=0$ only if ${\bf v}$ point in the opposite direction of ${\bf u}$.

Let us first consider the general case, i.e., ${\bf v}$ does not point in the opposite direction of ${\bf u}$. In view of \eqref{eqn:lagrangian}, let $g(\mu,\nu)=-\frac{1}{2\mu}\|{\bf v}+\nu{\bf u}\|^2-\nu d'-\frac{\mu r^2}{2}$. It is easy to see that
\begin{align}
\frac{\partial g(\mu,\nu)}{\partial \nu}=0\Leftrightarrow \nu = -\frac{{\bf v}^T{\bf u}+\mu d'}{\|{\bf u}\|^2}.
\end{align}
Since $\nu$ has to be nonnegative, we have
\begin{align}
\nu=\max\left\{0,-\frac{{\bf v}^T{\bf u}+\mu d'}{\|{\bf u}\|^2}\right\}.
\end{align}
{\bf Case 1.} If $-\frac{{\bf v}^T{\bf u}+\mu d'}{\|{\bf u}\|^2}\leq 0$, then $\nu=0$ and thus
\begin{align}
\frac{\partial g(\mu,\nu)}{\partial \mu}=0\Leftrightarrow \mu=\frac{\|{\bf v}\|}{r}.
\end{align}
Then $g(\mu,\nu)=-r\|{\bf v}\|$ and the optimal value of problem (\ref{prob:dome}) is given by
\begin{align}
{\bf v}^T{\bf o}-r\|{\bf v}\|.
\end{align}
{\bf Case 2.} If $-\frac{{\bf v}^T{\bf u}+\mu d'}{\|{\bf u}\|^2}> 0$, then $\nu=-\frac{{\bf v}^T{\bf u}+\mu d'}{\|{\bf u}\|^2}$ and
\begin{align}
\hspace{-4mm}g(\mu,\nu)=-\frac{1}{2\mu}\|{\bf v}^{\perp}\|^2-\frac{\mu}{2}\left(r^2-\frac{(d')^2}{\|{\bf u}\|^2}\right)+\frac{{\bf v}^T{\bf u}d'}{\|{\bf u}\|^2},
\end{align}
Thus,
\begin{align}
\frac{\partial g(\mu,\nu)}{\partial \mu}=0\Leftrightarrow \mu=\frac{\|{\bf v}^{\perp}\|}{\sqrt{r^2-\frac{(d')^2}{\|{\bf u}\|^2}}}
\end{align}
Then $g(\mu,\nu)=-\|{\bf v}^{\perp}\|\sqrt{r^2-\frac{(d')^2}{\|{\bf u}\|^2}}+\frac{{\bf v}^T{\bf u}d'}{\|{\bf u}\|^2}$ and the optimal value of problem (\ref{prob:dome}) is given by
\begin{align}
{\bf v}^T{\bf o}-\|{\bf v}^{\perp}\|\sqrt{r^2-\frac{(d')^2}{\|{\bf u}\|^2}}+\frac{{\bf v}^T{\bf u}d'}{\|{\bf u}\|^2}.
\end{align}

Now let us consider the case with ${\bf v}$ pointing in the opposite direction of ${\bf u}$. We can see that there exists $\gamma=-\frac{{\bf u}^T{\bf v}}{\|{\bf u}\|^2}>0$ such that ${\bf v}=-\gamma{\bf u}$. By plugging ${\bf v}=-\gamma{\bf u}$ in problem (\ref{prob:dome}) and following an analogous argument as before, we can see that the statement in Lemma \ref{lemma:dome} is also applicable to this case.

Therefore, the proof of the statement is completed.
\end{proof}

We are now ready to prove Theorem \ref{thm:snsv_i}.
\begin{proof}
To prove the statements in (\ref{snsv:rule1}) and (\ref{snsv:l}), we only need to set
\begin{align*}
&\begin{array}{lll}
{\bf v}:=\bar{\bf x}_i,&{\bf u}:-{\bf w}^*(s_a),&d:=-\|{\bf w}^*(s_a)\|^2,\\
{\bf o}:=\frac{1}{2}\hat{\bf w},&r:=\frac{1}{2}\|\hat{\bf w}\|,&
\end{array}\\
&d':=\rho=-\|{\bf w}^*(s_a)\|^2+\frac{1}{2}\langle{\bf w}^*(s_a),\hat{\bf w}\rangle,
\end{align*}
and then apply Lemma \ref{lemma:dome}. Notice that, for case 1, the optimal value
$$
f^*=\frac{1}{2}\left(\langle\hat{\bf w}(s_b),\bar{\bf x}_i\rangle-\|\hat{\bf w}(s_b)\|\|\bar{\bf x}_i\|\right)\leq0,
$$
and thus none of the non-support vectors can be identified [recall that, according to (\ref{rule1'}), $f^*$ has to be larger than $1$ such that $\bar{\bf x}_i$ can be detected as a non-support vector]. As a result, we only need to consider case 2.

The statement in (\ref{snsv:rule2}) and (\ref{snsv:u}) follows with an analogous argument by noting that
$$
\max_{{\bf w}\in\Theta'_{[s_b,s_a]}}\langle{\bf w},\bar{\bf x}_i\rangle=-\min_{{\bf w}\in\Theta'_{[s_b,s_a]}}-\langle{\bf w},\bar{\bf x}_i\rangle.
$$
\end{proof}

\end{document}